\documentclass[11pt]{article}
\usepackage{graphicx} 
\usepackage{bbm}
\usepackage[margin=1in]{geometry}
\usepackage{color}
\usepackage{amsfonts}
\usepackage{algorithm}
\usepackage{algorithmic}
\usepackage{latexsym, amssymb, amsmath, amscd, amsthm, amsxtra}
\usepackage{mathtools}
\usepackage{enumerate}
\usepackage[all]{xy}
\usepackage{mathrsfs}
\usepackage{fancyhdr}
\usepackage{listings}
\usepackage{hyperref}
\usepackage{enumitem}

\def\11{\mathbbm{1}}

\def\ER{Erd\H{o}s-R\'enyi\ }

\def\Fop{\operatorname{F}}

\newcommand{\Pb}{\mathbb P}
\newcommand{\Qb}{\mathbb Q}

\newtheorem{thm}{Theorem}[section]

\newtheorem{lemma}[thm]{Lemma}

\newtheorem{defn}[thm]{Definition}

\newtheorem{remark}[thm]{Remark}

\numberwithin{equation}{section}



\hypersetup{colorlinks=true,linkcolor=blue}
\title{A Computational Transition for Detecting Multivariate Shuffled Linear Regression by Low-Degree Polynomials}
\author{Zhangsong Li\\Peking University}
\date{}
\begin{document}
\maketitle

\begin{abstract}
    In this paper, we study the problem of multivariate shuffled linear regression, where the correspondence between predictors and responses in a linear model is obfuscated by a latent permutation. Specifically, we investigate the model $Y=\tfrac{1}{\sqrt{1+\sigma^2}}(\Pi_* X Q_* + \sigma Z)$, where $X$ is an $n*d$ standard Gaussian design matrix, $Z$ is an $n*m$ Gaussian noise matrix, $\Pi_*$ is an unknown $n*n$ permutation matrix, and $Q_*$ is an unknown $d*m$ on the Grassmanian manifold satisfying $Q_*^{\top} Q_* = \mathbb I_m$. 
    
    Consider the hypothesis testing problem of distinguishing this model from the case where $X$ and $Y$ are independent Gaussian random matrices of sizes $n*d$ and $n*m$, respectively. Our results reveal a phase transition phenomenon in the performance of \emph{low-degree polynomial algorithms} for this task. (1) When $m=o(d)$, we show that all degree-$D$ polynomials fail to distinguish these two models even when $\sigma=0$, provided with $D^4=o\big( \tfrac{d}{m} \big)$. (2) When $m=d$ and $\sigma=\omega(1)$, we show that all degree-$D$ polynomials fail to distinguish these two models provided with $D=o(\sigma)$. (3) When $m=d$ and $\sigma=o(1)$, we show that there exists a constant-degree polynomial that strongly distinguish these two models. These results establish a smooth transition in the effectiveness of low-degree polynomial algorithms for this problem, highlighting the interplay between the dimensions $m$ and $d$, the noise level $\sigma$, and the computational complexity of the testing task.
\end{abstract}



\section{Introduction}{\label{sec:intro}}

In this paper, we study the multivariate shuffled linear regression problem, which is rigorously defined as follows.
\begin{defn}{\label{def-shuffled-regression}}
    Consider the following linear model
    \begin{equation}{\label{eq-def-shuffled-linear-regression}}
        Y = \tfrac{1}{\sqrt{1+\sigma^2}} \big( \Pi_* X Q_* + \sigma Z \big) \,,
    \end{equation}
    here $X \in \mathbb R^{n*d}$ is the design matrix, $Q_* \in \mathbb R^{d*m}$ is the unknown regression matrix, $\Pi_*$ is an unknown permutation matrix that shuffles the rows of $X$, and $Z \in \mathbb R^{n*m}$ is the observation noise. We focus on the Bayesian setting where $X \in \mathbb R^{n*d}, Z \in \mathbb R^{n*m}$ are two independent random matrices with i.i.d.\ standard normal entries, $\Pi_*$ is a uniform $n*n$ permutation matrix, $Q_* \in \mathbb R^{d*m}$ is sampled from a prior distribution $\nu$ supported on all $d*m$ matrices and $X,Z,\Pi_*,Q_*$ are independent.
\end{defn} 

Two basic problems regarding this model are as follows: (1) the \emph{detection} problem, i.e., testing this model against two independent Gaussian matrices; (2) the \emph{estimation} problem, i.e., recovering $\Pi_*$ and $Q_*$ from the observations $(X,Y)$. If $\Pi_*$ is known, \eqref{eq-def-shuffled-linear-regression} reduces to standard linear regression. Otherwise, this problem is known as shuffled regression \cite{PWC17, APZ17, LWX24+}, unlabeled sensing \cite{UHV18, ZSL19, ZL20}, or linear regression with permuted/mismatched data \cite{SBD19, MW23, SBDL20}, as the correspondence between the predictors (the rows of $X$) and the responses (the rows of $Y$) is lost. Thus, it is a significantly more challenging problem as one needs to jointly estimate the permutation $\Pi_*$ and the regression coefficients $Q_*$. This problem has attracted considerable theoretical and practical interest due to its applications in areas such as robotics, data integration, and de-anonymization; for further details, we refer the readers to \cite[Section~1]{UHV18} and \cite[Section~1.1]{ZL20}. In this work we are interested in the algorithmic aspect of the detection problem, where the goal is to solve problem (1) using \emph{efficient algorithms} (i.e., algorithms with polynomial running time). Our results can be informally summarized as follows:

\begin{thm}[Informal]{\label{MAIN-THM-informal}}
    Suppose $m \leq d$ and $\nu$ is the Haar measure over the $d*m$ Grassmanian manifold $\{ Q \in \mathbb R^{d*m}: Q^{\top}Q=\mathbb I_m \}$. Then the following results hold:
    \begin{itemize}
        \item[(1)] when $m=o(d)$, all algorithms based on degree-$D$ polynomials (namely these algorithms are multivariate polynomials in the entries of $(X,Y)$ with degree bounded by $D$) fail to distinguish $(X,Y)$ from two independent Gaussian matrices (in the sense of Definition~\ref{def-strong-separation}) even when $\sigma=0$, provided that $D^4=o\big( \tfrac{d}{m} \big)$;
        \item[(2)] when $m=d$ and $\sigma=\omega(1)$, all algorithms based on degree-$D$ polynomials fail to distinguish $(X,Y)$ from two independent Gaussian matrices (in the sense of Definition~\ref{def-strong-separation}) provided that $D=o(\sigma)$;
        \item[(3)] when $m=d$ and $\sigma=o(1)$, there is an efficient algorithm based on degree-$O(1)$ polynomial that strongly distinguish $(X,Y)$ from two independent Gaussian matrices.  
    \end{itemize}
\end{thm}

\begin{remark} 
    In view of the widely accepted assumption that the failure of degree-$O(\log n)$ polynomials serves as evidence for computational hardness of the noisy problem (see, e.g., \cite[Section~2.2]{Hopkins18} and \cite[Section~3.3]{Wein25+}), our results suggest that the detection problem is computationally hard when $d/m \geq \operatorname{poly}(\log n),\sigma \geq 1/\operatorname{poly}(n)$ or when $d=m,\sigma\geq\operatorname{poly}(\log n)$. In addition, since we expect that estimation is at least as hard as detection, our results also suggest that the estimation problem is computationally hard when $d/m \geq \operatorname{poly}(\log n), \sigma \geq 1/\operatorname{poly}(n)$ or when $d=m,\sigma\geq\operatorname{poly}(\log n)$. This matches the known computational upper bound within poly-logarithmic factors, as when $d/m \geq \operatorname{poly}(\log n)$ the best known detection/estimation algorithm only works when $\sigma=0$ \cite{HSS17, AHSS17}, and when $d/m=O(1)$ the best known detection/estimation algorithm only works when $\sigma=O(1)$ \cite{NSX26+}.
    
    We also note that in the literature the shuffled linear regression problem \eqref{eq-def-shuffled-linear-regression} is usually stated in the non-Bayesian setting, where $Q_*$ is just an arbitrary unknown $d*m$ matrix. However, since in the singular value decomposition $Q_*=U_* \Lambda_* V_*$ where $V_* \in O(m)$ and $\Lambda_*$ is an $m*m$ diagonal matrix, assuming $V_*$ and $\Lambda_*$ is known to us will only simplify the problem (note that we can replace the observation $Y$ with $YV_*^{\top} \Lambda_*^{-}$ where $\Lambda_*^{-}$ is the pseudo-inverse of $\Lambda_*$). Thus, our hardness result (1) covers the most widely studied non-Bayesian setting and suggests that estimation is hard in this setting in the similar parameter regime.
\end{remark}

\begin{remark}
    After the submission of this manuscript, we were made aware of a forthcoming work \cite{NSX26+} that significantly strengthened Item~(3) of Theorem~\ref{MAIN-THM-informal}. The authors of \cite{NSX26+} prove that for $m \leq d,m=\Theta(d)$ and $\sigma=\Theta(1)$, there exists a polynomial of degree $D=\Theta(1)$ that strongly distinguishes $\Pb$ and $\Qb$; moreover, they show this polynomial can be leveraged into an efficient detection and estimation algorithm. Their results may be extended to give a degree $D=\Theta(\frac{d}{m})+\Theta(\sigma)+\Theta(1)$ polynomial that strongly distinguishes $\Pb$ and $\Qb$, provided $\frac{d}{m} \leq \operatorname{poly}(\log n)$ and $\sigma \leq \operatorname{poly}(\log n)$. Thus, in the regime $\frac{d}{m} \leq \operatorname{poly}(\log n)$, their upper bound matches our lower bound up to $\operatorname{poly}(\log n)$ factors. Together, these results characterize the smooth trade-off between the dimension ratio $\frac{d}{m}$ and the computational cost $D$ of this problem.
\end{remark}

\begin{remark}
    In the special case $m=d$, the upper bound in \cite{NSX26+} matches our lower bound, showing that degree-$D$ polynomials can strongly distinguish $\Pb$ and $\Qb$ if and only if $D\geq \Theta(\sigma)+\Theta(1)$ (provided that $\sigma \leq \operatorname{poly}(\log n)$). This suggests that $\sigma=O(1)$ is the separation between the computational ``easy'' and ``hard'' regime. Furthermore, it reveals a smooth trade-off between the noise $\sigma$ and the computational cost $D$ of this problem.
\end{remark}

\begin{remark}
    For general $\frac{d}{m} \geq \operatorname{poly}(\log n)$ and $\sigma=0$, we conjecture that a degree-$D$ polynomial can strongly distinguish $\Pb$ from $\Qb$ if and only if $D\geq \Theta(\frac{d}{m})$. Our current lower bound, which requires $D^4=o(\tfrac{d}{m})$ leaves a gap from this conjectured threshold of $D=\Theta(\tfrac{d}{m})$. This gap stems from potential $\operatorname{poly}(D)$ loss in the analysis of Lemma~\ref{lem-Adv-relax-1}, and we leave the goal for determining the sharp degree requirement for future work. We would like to remark that for the regime $\frac{d}{m} \gg \operatorname{poly}(\log n)$, our primary contribution is to provide evidence that no polynomial-time algorithm can solve the detection or estimation problem when $\sigma\geq \frac{1}{\operatorname{poly}(n)}$. For this purpose, the weaker lower bound of $D^4=o(\frac{d}{m})$ is already sufficient.
\end{remark}

\begin{remark}
    It is also natural to consider the regime $m>d$. In this case, $Q_*$ no longer lies on the Grassmanian manifold since the constraint $Q_*^{\top} Q_*=\mathbb I_m$ is infeasible. Therefore, it is more natural to consider the estimation problem with general $Q_*$. Note that this problem becomes easier when $m$ grows (since one can always consider the first $m' \leq m$ columns of $X$), the results in \cite{NSX26+} suggests that estimating $\Pi_*$ and $Q_*$ should be computationally feasible when $\sigma=O(1)$ (provided with some mild bound on $\|Q_*\|_{\Fop}$). 
    
    However, formulating a reasonable detection problem and proving its computational lower bound in this regime presents distinct challenges. The current low-degree polynomial framework is most effective for detection problems of a ``planted-versus-null'' nature. Thus, we need to assume that marginally $\{ Y_{i,j} \}$ are independent variables under $\Qb$ (and thus we also need to assume they are independent under $\Pb$, since otherwise detection is possible simply by checking the correlation between $Y$). This technical restriction forces us to assume that $Q_*^{\top} Q_*=\mathbb I_m$ and thus restricts us in the regime $m \leq d$. We believe that proving computational lower bound (for an appropriate detection problem) in the regime $m>d$ would require new ideas in the framework of low-degree polynomials, and thus we leave this as an important goal for future work.
\end{remark}

\subsection{Discussions}{\label{subsec:backgrounds}}

{\bf Information-computation gap and low-degree polynomial framework.} A special case of our result is the single-variate setting $m=1$, where $Q_* \in \mathbb R^d$ is sampled uniformly from the unit sphere $\mathbb S^{d-1}$. In this setting, the authors of \cite{LWX24+} determined the information threshold for recovering $\Pi_*$ given $(X,Y)$ when $d=o(n)$, which roughly speaking is $\sigma=O(n^{-2})$ for exact recovery and $\sigma=O(n^{-1})$ for almost-exact recovery. In addition, in a recent forthcoming work \cite{GWX25+} the authors determined the information threshold for the detection problem and the problem of estimating $Q_*$ up to a $\operatorname{poly}(d)$ factor when $d=o(\tfrac{n}{\log n})$, conjectured to be $\sigma^2=\Theta(\tfrac{1}{d\log d})$.\footnote{To be more precise, in \cite{GWX25+} the authors show that on the one hand, when $d=o(\tfrac{\log n}{\log\log n})$ the informational detection threshold is $\sigma^2 = o( \tfrac{1}{d \log d})$, and an efficient algorithm achieves the information threshold. On the other hand, for general $d=o( \tfrac{n}{\log n})$, detection is informationally possible if $\sigma^2= O(\tfrac{1}{(d \log d)^2})$ and informationally impossible if $\sigma^2=\omega(\tfrac{1}{d\log d})$. They also conjecture that $\sigma^2 = \Theta(\tfrac{1}{d \log d})$ is the true information threshold for detection.} However, when $d=\omega(\log n)$ currently no known polynomial-time algorithms that solve the detection or estimation problem except in the noiseless case $\sigma=0$ \cite{HSS17, AHSS17}. This suggests the presence of an \emph{information-computation gap}, a phenomenon commonly observed in high-dimensional statistical inference tasks \cite{ZK16, RSS19, KWB22, Gamarnik21}. Since at this point it seems rather elusive to prove hardness on the detection problem for a typical instance under the assumption of P$\neq$NP, we can only hope to prove hardness under even (much) stronger hypothesis. 

For this purpose, the \emph{low-degree polynomial framework} emerged as a powerful tool to provide \emph{evidence} for computational lower bounds. Indeed, it has been proved that the class of low-degree polynomial algorithms is a useful proxy for computationally efficient algorithms, in the sense that the best-known polynomial-time algorithms for a wide variety of high-dimensional inference problems are captured by the low-degree class such as spectral methods, approximate message passing and small subgraph counts; see e.g., \cite{Hopkins18, SW22, KWB22}. Furthermore, it is conjectured in \cite{Hopkins18} that the failure of degree-$D$ polynomial algorithms implies the failure of all ``robust'' algorithms with running time $n^{\widetilde{O}(D)}$ (here $\widetilde{O}$ means having at most this order up to a $\operatorname{poly}(\log n)$ factor). Consequently, the failure of degree-$O(\operatorname{poly}(\log n))$ polynomial algorithms to solve a problem is often interpreted as evidence of its intrinsic average-case hardness, or at least it is expected that breaking such impossibility results would require a major breakthrough in algorithms. While it is important to note that recent work \cite{BHJK25} has uncovered counterexamples to this conjecture, the framework remains a highly valuable tool for predicting average-case computational lower bounds. Thus, our results provide further evidence supporting the existence of an information-computation gap in this problem when $d \geq \operatorname{poly}(\log n)$.

\

\noindent{\bf Lattice-based algorithm ``beats'' low-degree polynomials.} Although in the high-dimensional, single-variate case $m=1,d=\operatorname{poly}(\log n)$ we provide evidence that low-degree polynomials fail at the detection task even in the noiseless case $\sigma=0$, we note that there is a lattice-based polynomial-time algorithm solving the detection and estimation tasks when $\sigma=0$ \cite{HSS17,AHSS17}. Thus, the shuffled regression problem is a rare example where a lattice-based algorithm outperforms the family of low-degree polynomial algorithms. A similar phenomenon has been observed in the problem of clustering a variant of Gaussian mixtures, as discovered in \cite{DK22, ZSWB22}.

However, as discussed in \cite{ZSWB22}, this phenomenon does not necessary contradict the low-degree conjecture proposed in \cite{Hopkins18}. Indeed, the lattice-based algorithms in both \cite{ZSWB22} and \cite{HSS17, AHSS17} are highly sensitive to the specifics of the model, and rely on the observations being ``noiseless'' in some sense. In fact, these algorithms appear to fail as long as the noise level satisfies $\sigma=\exp(-\operatorname{poly}(n))$. Therefore, our low-degree hardness results likely indicate that all ``robust'' algorithms fail to solve the \emph{noisy} detection/estimation problem even when the noise $\sigma$ is polynomially small.

Nonetheless, this serves as an important reminder that low-degree polynomials can sometimes be surpassed. We therefore clarify that the low-degree framework is expected to be optimal for a certain, yet imprecisely defined, class of ``high-dimensional'' problems. Recent work \cite{BHJK25} has further shown that the low-degree conjecture (posed in \cite{Hopkins18} and refined in \cite{HW21}) can be refuted. Despite these important caveats, we maintain that analyzing low-degree polynomials remains highly meaningful for our setting, as it provides compelling evidence for average-case hardness by ruling out a broad class of powerful algorithmic approaches. We refer the reader to the survey \cite{Wein25+} for a more detailed discussion of these subtleties. 

\

\noindent{\bf Procrustes matching and geometric graph matching.} Another special case of our results which may be of particular interest is the setting where $m=d$ and $Q_*$ is uniformly sampled from the family of $d*d$ orthogonal matrices. This problem, known as \emph{Procrustes matching}, has numerous applications in fields such as natural language processing and computer vision \cite{RCB97, MDK+16, DL17, GJB19}. 

In addition, it was shown in \cite{WWXY22} that this problem is further equivalent to a geometric variant of the random graph alignment problems \cite{WWXY22, GL24+, EGMM24}. More precisely, consider the Gaussian model
\begin{align*}
    Y=\Pi_* X + \sigma Z \,, \quad \mbox{ where } X,Y,Z \in \mathbb R^{n*d} 
\end{align*}
such that $X_i,Z_i$ being i.i.d.\ $\mathcal N(0,\mathbb I_d)$ vectors. The observation is two correlated Wishart matrices $A=XX^{\top}$ and $B=YY^{\top}$. In the low-dimensional regime where $d=O(\log n)$, the information threshold for estimating $\Pi_*$ was determined in \cite{WWXY22} and an efficient algorithm approaching the information threshold up to a $\operatorname{poly}(d)$ factor was found in \cite{GL24+}. In contrast, the related detection problem was less understood, especially in the high-dimensional regime $d=\omega(\log n)$. Our work shows the computation transition phenomenon for this detection problem under the low-degree polynomial framework, suggesting that (roughly speaking) $\sigma=O(1)$ is the separation of the computational ``easy'' and ``hard'' regime.

\ 

\noindent{\bf Open problems.} While we have shown the computational lower bound of the detection problem in the low-degree polynomial framework, the information-theoretic threshold for this detection problem in multivariate case remains largely open. In the forthcoming work \cite{GWX25+}, inspired by the analysis of the broken sample problem \cite{DCK20, KW22, JWX25+} the authors show that in the single-variate case, the detection threshold is given by $\sigma^2=\Theta(\tfrac{1}{d \log d})$ when $d=o(\tfrac{\log n}{\log\log n})$ and it lies in the range $\Theta(\tfrac{1}{d^2 (\log d)^2}) \leq \sigma^2 \leq \Theta(\tfrac{1}{d \log d})$ when $\tfrac{\log n}{\log\log n} \ll d \ll \tfrac{n}{\log n}$. Extending these results to the multivariate setting presents an intriguing open problem.

Another interesting question left open by our work is whether we could provide additional evidence on the information-computation gaps emerged in this detection problem when $d/m\geq\operatorname{poly}(\log n)$. For instance, it would be interesting if we could establish the sum-of-squares lower bound \cite{GJJ+20, JPRX23} for this problem. In addition, we note that in the problem of clustering a variant of Gaussian mixtures mentioned above, the hardness of the noisy problem can be shown by reducing to the standard assumption \cite[Conjecture~1.2]{MR09} from lattice-based cryptography that certain worst-case lattice problems are hard against quantum algorithms \cite{BRST21}. However, it remains unclear whether similar reductions exist for our problem. Finally, one might wonder if the hardness of shuffled regression can be related to other well-studied problems, such as the (seemingly similar) Non-Gaussian Component Analysis (NGCA) \cite{DKPP24}. However, the latent permutation $\Pi_*$ introduces complex dependencies between samples (in all directions), creating significant obstacles for establishing a rigorous reduction. Understanding the precise relationship between these problems is another challenging avenue for future work.

\subsection{Notations}{\label{subsec:notations}}
 
We record in this subsection some notation conventions. Denote $\mu$ to be the uniform distribution over the set of all permutation matrices $\mathfrak S_n$ and denote $\nu$ to be the Haar distribution over the $d*m$ Grassmanian manifold. In addition, let $\mathbb S^{d-1}$ be the unit $d$-dimensional sphere. For a matrix or a vector $M$, we will use $M^{\top}$ to denote its transpose. Denote $O(m)$ to be the set of all $m*m$ orthogonal matrices. For a $k*k$ matrix $M=(m_{ij})_{k*k}$, let $\operatorname{det}(M)$ and $\operatorname{tr}(M)$ be the determinant and trace of $M$, respectively. Denote $M \succ 0$ if $M$ is positive definite and $M \succeq 0$ if $M$ is semi-positive definite. In addition, if $M$ is symmetric we let $\varsigma_1(M) \geq \varsigma_2(M) \geq \ldots \geq \varsigma_k(M)$ be the eigenvalues of $M$. Denote by $\mathrm{rank}(M)$ the rank of the matrix $M$. For two $k*l$ matrices $M_1$ and $M_2$, we define their inner product to be
\begin{align*}
    \big\langle M_1,M_2 \big\rangle:=\sum_{i=1}^k \sum_{j=1}^l M_1(i,j)M_2(i,j) \,.
\end{align*}
We also define the Frobenius norm, operator norm, and $\infty$-norm of $M$ respectively by
\begin{align*}
    \| M \|_{\operatorname{F}} = \mathrm{tr}(MM^{\top}) =  \langle M,M \rangle^{\frac{1}{2}}, \
    \| M \|_{\operatorname{op}} = \varsigma_1(M M^{\top})^{\frac{1}{2}}, \
    \| M \|_{\infty} = \max_{ \substack{ 1 \leq i \leq k \\ 1 \leq j \leq l } } |M_{i,j}| 
\end{align*}
where $\mathrm{tr}(\cdot)$ is the trace for a squared matrix. We will use $\mathbb{I}_{k}$ to denote the $k*k$ identity matrix (and we drop the subscript if the dimension is clear from the context). Similarly, we denote $\mathbb{O}_{k*l}$ the $k*l$ zero matrix and denote $\mathbb{J}_{k*l}$ the $k*l$ matrix with all entries being 1. We will abbreviate $\mathbb O_k = \mathbb O_{k*1}$ and $\mathbb J_k=\mathbb J_{k*1}$.  

For any $\alpha=(\alpha_1,\ldots,\alpha_k)\in\mathbb N^k$, define $|\alpha|=\sum_{1 \leq i \leq k} \alpha_i$ and $\alpha!=\alpha_1!\ldots\alpha_k!$. In addition, for any $c \in \mathbb N$ and $x \in \mathbb R^k$, define $c\alpha=(c\alpha_1,\ldots,c\alpha_k)$ and $x^{\alpha} = x_1^{\alpha_1} \ldots x_k^{\alpha_k}$. For $m\in \mathbb N$ and $\alpha\in\mathbb N^k$ with $|\alpha|=m$, denote $\binom{m}{\alpha}=\frac{m!}{\alpha_1!\ldots \alpha_d!}$. We denote $\mathsf A \in (\mathbb N^{k})^{\otimes n}$, if $\mathsf A=(\alpha_i: i \in [n])$ such that $\alpha_i \in \mathbb N^k$. For two $\mathsf A,\mathsf A' \in (\mathbb N^{k})^{\otimes n}$ with $\mathsf A=(\alpha_i: i \in [n])$ and $\mathsf A'=(\alpha'_i: i \in [n])$, we say $\mathsf A \cong \mathsf A'$ if and only if there exists a bijection $\pi:[n]\to [n]$ such that $\alpha_{i}=\alpha'_{\pi(i)}$ for all $i \in [n]$. Denote $\mathsf{Aut}(\mathsf A)$ to be the number of bijections $\pi:[n] \to [n]$ such that $\alpha_i=\alpha_{\pi(i)}$ for all $i \in [n]$. For a set $A$, we use both $|A|$ and $\#A$ to denote its cardinality. The indicator function of sets $A$ is denoted by $\mathbf{1}_{A}$. For two probability measures $\mathbb P$ and $\mathbb Q$, we denote the total variation distance between them by $\mathsf{TV}(\mathbb P,\mathbb Q)$. The chi-squared divergence from $\Pb$ to $\Qb$ is defined as $\chi^2(\Pb \| \Qb)= \mathbb E_{\mathbf{X}\sim\Qb}[ (\frac{\mathrm{d}\Pb}{\mathrm{d}\Qb}(\mathbf X))^2 ]$.

For any two positive sequences $\{a_n\}$ and $\{b_n\}$, we write equivalently $a_n=O(b_n)$, $b_n=\Omega(a_n)$, $a_n\lesssim b_n$ and $b_n\gtrsim a_n$ if there exists a positive absolute constant $c$ such that $a_n/b_n\leq c$ holds for all $n$. We write $a_n=o(b_n)$, $b_n=\omega(a_n)$, $a_n\ll b_n$, and $b_n\gg a_n$ if $a_n/b_n\to 0$ as $n\to\infty$. We write $a_n =\Theta(b_n)$ if both $a_n=O(b_n)$ and $a_b=\Omega(b_n)$ hold.

\section{The low-degree polynomial framework}{\label{subsec:ldp-framework}}

In this section we provide the formal statement of Theorem~\ref{MAIN-THM-informal}. More precisely, we will consider the following hypothesis testing problem, where
\begin{itemize}
    \item under the null hypothesis $\mathcal H_0$, we let $X \in \mathbb R^{n*d}$ and $Y \in \mathbb R^{n*m}$ be two independent Gaussian matrices with i.i.d.\ standard normal entries, respectively;
    \item under the alternative hypothesis $\mathcal H_1$, we let $(X,Z,\Pi_*,Q_*)$ be independent random matrices with $X \in \mathbb R^{n*d},Z \in \mathbb R^{n*m}$ having i.i.d.\ standard normal entries, $\Pi_*$ an $n*n$ uniform permutation matrix and $Q_*$ sampled from $\nu$, and let $Y=\tfrac{1}{\sqrt{1+\sigma^2}} (\Pi_* X Q_*+\sigma Z)$.
\end{itemize}
The normalization factor $\tfrac{1}{\sqrt{1+\sigma^2}}$ (together with the assumption $Q_*\sim\nu$) ensures that the entries $\{ Y_{i,j} \}$ are marginally distributed as i.i.d.\ standard normal variables under both hypotheses. We establish evidence the computational hardness of this hypothesis testing problem by focusing on a specific class of algorithms known as \emph{low-degree polynomials}. Inspired by the sum-of-squares hierarchy, the low-degree polynomial method offers a promising approach for deriving computational lower bounds in high-dimensional inference problems. At its core, this method analyzes algorithms that rely on evaluating collections of polynomials with moderate degrees. The study of this category of algorithms is motivated by the examination of high-dimensional hypothesis testing problems \cite{BHK+19, Hopkins18, HKP+17, HS17}, with an extensive overview provided in \cite{KWB22}. A key strength of the low-degree polynomial approach is its ability to deliver tight hardness results for a wide range of problems. Notable examples include detection problems such as planted clique, planted dense subgraph, community detection, sparse PCA, graph alignment (see \cite{HS17, HKP+17, Hopkins18, KWB22, SW22, DMW23+, BKW20, DKW+22, MW22+, DDL23+, KMW24, CDGL24+}), optimization problems such as maximal independent sets in sparse random graphs \cite{GJW20, Wein22}, and constraint satisfaction problems such as random $k$-SAT \cite{BB22}. In the remaining of this paper, we will focus on applying this framework in the context of hypothesis testing problems.

More precisely, denote $\Pb=\Pb_n$ to be the law of $(X,Y)$ under $\mathcal H_1$ and $\Qb=\Qb_n$ to be the law of $(X,Y)$ under $\mathcal H_0$. In addition, let $\mathbb R[X,Y]_{\leq D}$ denote the set of multivariate polynomials in the entries of $(X,Y)$ with degree at most $D$. With a slight abuse of notation, we will often say ``a polynomial'' to mean a sequence of polynomials $f=f_n \in \mathbb R[X,Y]_{\leq D}$, one for each problem size $n$; the degree $D=D_n$ of such a polynomial may scale with $n$. To probe the computational threshold for testing between two sequences of probability measures $\Pb$ and $\Qb$, we consider the following notions of strong separation and weak separation defined in \cite[Definition~1.6]{BAH+22}.

\begin{defn}{\label{def-strong-separation}}
    Let $f \in \mathbb{R}[X,Y]_{\leq D}$ be a polynomial.
    \begin{itemize}
        \item We say $f$ strongly separates $\Pb_n$ and $\Qb_n$ if as $n \to \infty$ 
        \begin{align*}
            \sqrt{ \max\big\{ \operatorname{Var}_{\Pb}(f(A,B)), \operatorname{Var}_{\Qb}(f(A,B)) \big\} } = o\big( \big| \mathbb{E}_{\Pb}[f(A,B)] - \mathbb{E}_{\Qb}[f(A,B)] \big| \big) \,;
        \end{align*}
        \item We say $f$ weakly separates $\Pb_n$ and $\Qb_n$ if as $n \to \infty$ 
        \begin{align*}
            \sqrt{ \max\big\{ \operatorname{Var}_{\Pb}(f(A,B)), \operatorname{Var}_{\Qb}(f(A,B)) \big\} } = O\big( \big| \mathbb{E}_{\Pb}[f(A,B)] - \mathbb{E}_{\Qb}[f(A,B)] \big| \big) \,.
        \end{align*}
    \end{itemize}
\end{defn}
See \cite{BAH+22} for a detailed discussion of why these conditions are natural for hypothesis testing. In particular, according to Chebyshev's inequality, strong separation implies that we can threshold $f(A,B)$ to test $\Pb$ against $\Qb$ with vanishing type-I and type-II errors. Now we can state our rigorous statement justifying Theorem~\ref{MAIN-THM-informal}.

\begin{thm}{\label{main-thm-lower-bound}}
    \begin{enumerate}
        \item[(1)] Suppose $1 \leq m=o(d)$ and $D^4=o\big(\tfrac{d}{m}\big)$. Then no degree-$D$ polynomial weakly separates $\Pb$ and $\Qb$.
        \item[(2)] Suppose $m=d,\sigma=\omega(1)$ and $D=o(\sigma)$. Then no degree-$D$ polynomial weakly separates $\Pb$ and $\Qb$. 
        \item[(3)] Suppose $m=d$ and $\sigma=o(1)$. Then there is a degree-$O(1)$ polynomial that strongly separates $\Pb$ and $\Qb$.
    \end{enumerate}
\end{thm}

\section{Proof of Theorem~\ref{main-thm-lower-bound}}

This section is devoted to the proof of Theorem~\ref{main-thm-lower-bound}. The primary challenge lies in establishing the computational lower bounds in Items (1) and (2). We now outline our proof strategy, which centers on bounding the low-degree advantage defined in \eqref{eq-def-low-degree-adv}. Our approach proceeds in three main steps:
\begin{itemize}
    \item First, we exploit the fact that $\Qb$ is a product measure to derive an explicit bound for the low-degree advantage using an orthogonal basis under $\Qb$.
    \item The resulting expression is difficult to analyze directly. Our key innovation is to relate this low-degree advantage between $\Pb,\Qb$ to the chi-square divergence between two simpler measures $\mathtt P,\mathtt Q$, defined in Definition~\ref{def-easier-hypothesis-testing}.
    \item Finally, we compute the required chi-square divergence using techniques from random matrix theory, which yields the desired bound on the low-degree advantage.
\end{itemize}
We believe this technique of reducing a complex low-degree advantage to a simpler chi-square divergence is of independent interest and may be applicable to other permutation-based inference problems.

The remainder of this section is organized as follows. Section~\ref{subsec:universal-Adv-bound} derives a universal upper bound for the low-degree advantage using an orthogonal basis. In Section~\ref{subsec:relating-linear-model} we relate this bound to the chi-square divergence between two simpler measures. Sections~\ref{subsec:proof-main-thm-lower-bound-2} and \ref{subsec:proof-main-thm-lower-bound-3} bound this chi-square divergence under the assumptions of Items~(1) and (2), respectively, thereby completing their proofs. Finally, in Section~\ref{subsec:proof-main-thm-upper-bound} we prove Item~(3) of Theorem~\ref{main-thm-lower-bound}.

\subsection{Explicit bound on low-degree advantage}{\label{subsec:universal-Adv-bound}}

We begin with the classical framework for proving the failure of weak separation of polynomials, as introduced in \cite{HS17}.

\begin{defn}{\label{def-low-degree-adv}}
    For a quadruple $(n,d,m,D)$ with $d,m,D$ possibly depend on $n$, we define the \emph{low-degree advantage} by
    \begin{equation}{\label{eq-def-low-degree-adv}}
        \mathsf{Adv}_{\leq D}(\Pb,\Qb):= \sup_{ f \in \mathbb R[X,Y]_{\leq D} } \frac{ \mathbb E_{\Pb}[f] }{ \sqrt{ \mathbb E_{\Qb}[f^2] } } \,.
    \end{equation}
    It is clear (see \cite[Definition~1.8]{BAH+22}) that $\mathsf{Adv}_{\leq D}(\Pb,\Qb)=1+o(1)$ as $n\to\infty$ implies that no polynomial in $\mathbb R[X,Y]_{\leq D}$ weakly separates $\Pb$ and $\Qb$.
\end{defn}

The rest of this section is devoted to proving $\mathsf{Adv}_{\leq D}(\Pb,\Qb)=1+o(1)$ under various settings. The following polynomials will play a fundamental role in our analysis.

\begin{defn}{\label{def-Hermite-poly}}
    For all $m \in \mathbb N$, define the Hermite polynomials by
    \begin{equation}{\label{eq-def-Hermite-poly}}
        H_0(z)=1 \,, \quad H_1(z)=z \,, \quad H_{m+1}(z)= zH_m(z)-mH_{m-1}(z) \,.
    \end{equation}
    In addition, define $\mathcal H_m(z)=\frac{1}{\sqrt{m!}} H_m(z)$, and for all $x \in \mathbb R^d$ and $\alpha\in \mathbb N^{d}$ define
    \begin{equation}{\label{eq-def-mathcal-H}}
        \mathcal H_{\alpha}(x) := \prod_{i=1}^{d} \mathcal H_{\alpha_i}(x_i) \,.
    \end{equation}
    For all $y \in \mathbb R^m$ and $\beta\in \mathbb N^m$, define $\mathcal H_{\beta}(y)$ in the similar manner. For all $\mathsf A=(\alpha_i: i \in [n]) \in (\mathbb N^d)^{\otimes n}$ and $\mathsf B=(\beta_j:j \in [n]) \in (\mathbb N^m)^{\otimes n}$ such that $\alpha_i\in \mathbb N^d,\beta_j \in \mathbb N^m$, define
    \begin{equation}{\label{eq-def-phi-alpha,beta}}
        \phi_{\mathsf A,\mathsf B}(X,Y) = \prod_{1\leq i \leq n} \mathcal H_{\alpha_i}(X_i) \prod_{1 \leq j \leq n} \mathcal H_{\beta_j}(Y_j) \,. 
    \end{equation}
\end{defn}

Define $|\mathsf A|=\sum_{1 \leq i \leq n} |\alpha_i|$ and $|\mathsf B|=\sum_{1 \leq j \leq n} |\beta_j|$. It is well known (see, e.g., \cite{Szego39}) that $\{ \phi_{\mathsf A,\mathsf B}: |\mathsf A|+|\mathsf B| \leq D \}$ forms a standard orthogonal basis of $\mathbb R[X,Y]_{\leq D}$ under the measure $\Qb$, i.e., we have
\begin{equation}{\label{eq-standard-orthogonal}}
    \mathbb E_{\Qb}\big[ \phi_{\mathsf A,\mathsf B} \phi_{\mathsf A',\mathsf B'} \big] = \mathbf 1_{ \{ (\mathsf A,\mathsf B)=(\mathsf A',\mathsf B') \} } \,.
\end{equation}
Consequently, we can explicitly express the low-degree advantage $\mathsf{Adv}_{\leq D}(\Pb,\Qb)$ in terms of this basis, as in the next lemma.

\begin{lemma}{\label{lem-Adv-transform}}
    Recall \eqref{eq-def-low-degree-adv}. For any $n,D\geq 1$, it holds that
    \begin{equation}{\label{eq-Adv-transform}}
        \mathsf{Adv}_{\leq D}(\Pb,\Qb) = \Bigg( \sum_{ (\mathsf A,\mathsf B): |\mathsf A|+|\mathsf B|\leq D } \mathbb E_{\Pb}\Big[ \phi_{\mathsf A,\mathsf B}(X,Y) \Big]^2 \Bigg)^{1/2} \,.
    \end{equation}
\end{lemma}
\begin{proof}
    For any $f \in \mathbb R[X,Y]_{\leq D}$, it can be uniquely expressed as
    \begin{align*}
        f=\sum_{ (\mathsf A,\mathsf B): |\mathsf A|+|\mathsf B|\leq D } C_{\mathsf A,\mathsf B} \phi_{\mathsf A,\mathsf B} \,,
    \end{align*}
    where $C_{\mathsf A,\mathsf B}$'s are real constants. Applying Cauchy-Schwartz inequality one gets
    \begin{align*}
        \frac{ \mathbb{E}_{\mathbb{P}}[f] }{ \sqrt{\mathbb{E}_{\mathbb{Q}}[f^2]} } = \frac{ \sum_{ (\mathsf A,\mathsf B): |\mathsf A|+|\mathsf B|\leq D } C_{\mathsf A,\mathsf B} \mathbb{E}_{\mathbb{P}}[\phi_{\mathsf A,\mathsf B}] }{ \sqrt{ \sum_{ (\mathsf A,\mathsf B): |\mathsf A|+|\mathsf B|\leq D } C_{\mathsf A,\mathsf B}^2} } \leq \Bigg( \sum_{ (\mathsf A,\mathsf B): |\mathsf A|+|\mathsf B|\leq D }\big(\mathbb E_{\Pb} [\phi_{\mathsf A,\mathsf B}]\big)^2 \Bigg)^{1/2} \,,
    \end{align*}
    with equality holds if and only if $C_{\mathsf A,\mathsf B} \propto \mathbb{E}_\Pb[\phi_{\mathsf A,\mathsf B}]$. 
\end{proof}

\begin{remark}
    A quantity related to, but stronger than, the low-degree advantage in \eqref{eq-def-low-degree-adv} is the \emph{low-coordinate advantage}, defined as
    \begin{align}{\label{eq-def-low-coordinate-Adv}}
        \mathsf{CorAdv}_{\leq D}(\Pb,\Qb):= \sup_{ \mathsf{Cor}(f) \leq D } \frac{ \mathbb E_{\Pb}[f] }{ \sqrt{ \mathbb E_{\Qb}[f^2] } } \,.
    \end{align}
    Here the coordinate degree $\mathsf{Cor}(f)$ is the smallest integer $K$ such that the function $f$ can be expressed as a sum of functions, each of which depends on at most $K$ coordinates from $\{ X_{i,j} \}$ and $\{ Y_{i,j} \}$. This quantity was introduced in \cite{Hopkins18} and recently studied in \cite{Kun25}. We can develop a similar closed-form expression for $\mathsf{CorAdv}_{\leq D}(\Pb,\Qb)$ in this model. Let $\| \mathsf A \|_0, \| \mathsf B \|_0$ be the number of non-zero entries of $\mathsf A,\mathsf B$ respectively. Then, the set $\{ \phi_{\mathsf A,\mathsf B}: \| \mathsf A \|_0+ \| \mathsf B \|_0 \leq D \}$ forms an orthonormal basis for the linear space of functions with coordinate degree at most $D$. Using an argument analogous to the proof of Lemma~\ref{lem-Adv-transform}, we can show that
    \begin{align}
        \mathsf{CorAdv}_{\leq D}(\Pb,\Qb) = \Bigg( \sum_{ (\mathsf A,\mathsf B): \| \mathsf A \|_0 + \| \mathsf B \|_0 \leq D } \mathbb E_{\Pb}\big[ \phi_{\mathsf A,\mathsf B}(X,Y) \big]^2 \Bigg)^{1/2} \,.  \label{eq-CorAdv-transform}
    \end{align}
\end{remark}

The main part of this section is devoted to bounding the right hand side of \eqref{eq-Adv-transform} under various assumptions in Theorem~\ref{main-thm-lower-bound}. For $\alpha \in \mathbb N^d, \beta \in \mathbb N^m$ and $Q \in \mathbb R^{d*m}$, denote 
\begin{align}{\label{eq-def-Lambda-alpha,beta}}
    \Lambda_{\alpha,\beta}(Q) := \mathbb E_{U,V}\Big[ \mathcal H_{\alpha}\big( U \big) \mathcal H_{\beta}\big( \tfrac{1}{\sqrt{1+\sigma^2}}(UQ+\sigma V) \big) \Big] \,, \mbox{ where } U \perp V, U\sim \mathcal N(0,\mathbb I_d), V \sim \mathcal N(0,\mathbb I_m) \,.
\end{align}
(Here we treat $U,V$ as row vectors.) We first bound the right hand side of \eqref{eq-Adv-transform} by $\Lambda_{\alpha,\beta}(Q)$'s.

\begin{lemma}{\label{lem-Adv-relax-1}}
    We have
    \begin{align}{\label{eq-Adv-relax-1}}
        \sum_{ \substack{ |\mathsf A|+|\mathsf B|\leq D } } \mathbb E_{\Pb}\big[ \phi_{\mathsf A,\mathsf B}(X,Y) \big]^2 \leq 1 + \sum_{1 \leq k \leq D} \sum_{ \substack{ (\alpha_1,\ldots,\alpha_k): \alpha_i \in \mathbb N^d \setminus \{ \mathbb O_d \} \\ (\beta_1,\ldots,\beta_k): \beta_i \in \mathbb N^m \setminus \{ \mathbb O_m \} \\ |\alpha_1|+|\beta_1|+\ldots+|\alpha_k|+|\beta_k| \leq D } } \mathbb E_{Q \sim \nu}\Big[ \prod_{1 \leq i \leq k} \Lambda_{\alpha_i,\beta_i}(Q) \Big]^2 \,.
    \end{align}
\end{lemma}
\begin{proof}
    Denote $\Pb_{\pi}=\Pb(\cdot\mid\pi_*=\pi)$. We first give an explicit expression for $\mathbb E_{\Pb_\pi}[ \phi_{\mathsf A,\mathsf B} ]$. For all $\mathsf A=(\alpha_i:i\in [n]) \in (\mathbb N^d)^{\otimes n}$ and $\mathsf B=(\beta_i:i\in [n]) \in (\mathbb N^m)^{\otimes n}$, using \eqref{eq-def-Lambda-alpha,beta} we see that 
    \begin{align*}
        \mathbb E_{\Pb_\pi}\big[ \phi_{\mathsf A,\mathsf B} \mid Q \big] = \prod_{1 \leq i \leq n} \mathbb E_{X_i \sim \mathcal N(0,\mathbb I_d)} \Big[ \mathcal H_{\alpha_i}(X_i) \mathcal H_{\beta_{\pi(i)}}(X_i Q) \Big] = \prod_{1 \leq i \leq n} \Lambda_{\alpha_i,\beta_{\pi(i)}}(Q) \,,  
    \end{align*}
    and thus
    \begin{align}
        \mathbb E_{\Pb_\pi}\big[ \phi_{\mathsf A,\mathsf B} \big] = \mathbb E_{Q \sim \nu}\Big[ \prod_{1 \leq i \leq n} \Lambda_{\alpha_i,\beta_{\pi(i)}}(Q) \Big] \,. \label{eq-conditioned-on-pi}
    \end{align}
    We now return to the proof of Lemma~\ref{lem-Adv-relax-1}. To this end, for all $\mathsf A \in (\mathbb N^m)^{\otimes n}$ and $\mathsf B \in (\mathbb N^d)^{\otimes n}$, we can characterize $(\mathsf A,\mathsf B)$ with two subsets $S,K \subset [n]$ and two ordered families of integer vectors $(\alpha_1,\ldots,\alpha_{|S|}), (\beta_1, \ldots, \beta_{|K|})$ where $\alpha_i \in \mathbb N^d, \beta_j \in \mathbb N^m$ and $\alpha_i \neq \mathbb O_d$ and $\beta_j \neq \mathbb O_m$. In addition, since $|\mathsf A|+|\mathsf B|\leq D$ we must have 
    \begin{align*}
        |S|,|K| \leq D \mbox{ and } \sum_{1 \leq i \leq |S|} |\alpha_i| + \sum_{1\leq j \leq |K|} |\beta_j| \leq D \,.
    \end{align*}
    Since \eqref{eq-def-Lambda-alpha,beta} yields that
    \begin{align*}
        \Lambda(\alpha,\beta)=0 \mbox{ for } \alpha=\mathbb O_d, \beta \neq \mathbb O_m \mbox{ or } \alpha \neq \mathbb O_d, \beta = \mathbb O_m \,,
    \end{align*}
    we have $\mathbb E_{\Pb_\pi}[ \phi_{\mathsf A,\mathsf B} ] = 0$ if $\pi(S) \neq K$. When $\pi(S)=K$, denote $\sigma$ be the restriction of $\pi$ on $S$, we then have (note that \eqref{eq-def-Lambda-alpha,beta} yields $\Lambda(\mathbb O_d,\mathbb O_m)=1$)
    \begin{align*}
        \mathbb E_{\Pb_\pi}\big[ \phi_{\mathsf A,\mathsf B} \big] = \mathbb E_{Q \sim \nu}\Big[ \prod_{1 \leq i \leq |S|} \Lambda_{\alpha_i,\beta_{\sigma(i)}}(Q) \Big] \,. 
    \end{align*}
    Thus, we have shown that
    \begin{align*}
        \mathbb E_{\Pb}\big[ \phi_{\mathsf A,\mathsf B} \big] = 
        \begin{cases}
            0 \,, & |S| \neq |K| \,; \\
            1 \,, & S=K=\emptyset \,; \\
            \frac{(n-k)!}{n!} \sum_{\sigma \in \mathfrak S_k} \mathbb E_{Q \sim \nu}\Big[ \prod_{1 \leq i \leq k} \Lambda_{\alpha_i,\beta_{\sigma(i)}}(Q) \Big] \,, & |S|=|K|=k \,.
        \end{cases}
    \end{align*}
    Thus, we have the left hand side of \eqref{eq-Adv-relax-1} is bounded by
    \begin{align}
        & 1+ \sum_{1 \leq k \leq D} \sum_{ \substack{ S,K \subset [n] \\ |S|=|K|=k } } \sum_{ \substack{ (\alpha_1,\ldots,\alpha_k): \alpha_i \in \mathbb N^d \setminus \{ \mathbb O_d \} \\ (\beta_1,\ldots,\beta_k): \beta_i \in \mathbb N^m \setminus \{ \mathbb O_m \} \\ |\alpha_1|+|\beta_1|+\ldots+|\alpha_k|+|\beta_k| \leq D } }  \Bigg( \frac{(n-k)!}{n!} \sum_{\sigma \in \mathfrak S_k} \mathbb E_{Q \sim \nu}\Big[ \prod_{1 \leq i \leq k} \Lambda_{\alpha_i,\beta_{\sigma(i)}}(Q) \Big] \Bigg)^2 \nonumber \\
        \leq\ & 1+ \sum_{1 \leq k \leq D} \sum_{ \substack{ (\alpha_1,\ldots,\alpha_k): \alpha_i \in \mathbb N^d \setminus \{ \mathbb O_d \} \\ (\beta_1,\ldots,\beta_k): \beta_i \in \mathbb N^m \setminus \{ \mathbb O_m \} \\ |\alpha_1|+|\beta_1|+\ldots+|\alpha_k|+|\beta_k| \leq D } }  \Bigg( \frac{1}{k!} \sum_{\sigma \in \mathfrak S_k} \mathbb E_{Q \sim \nu}\Big[  \prod_{1 \leq i \leq k} \Lambda_{\alpha_i,\beta_{\sigma(i)}}(Q) \Big] \Bigg)^2 \,. \label{eq-Adv-relax-2}
    \end{align}
    Using Cauchy-Schwartz inequality, we have \eqref{eq-Adv-relax-2} is bounded by
    \begin{align*}
         & 1+ \sum_{1 \leq k \leq D} \sum_{ \substack{ (\alpha_1,\ldots,\alpha_k): \alpha_i \in \mathbb N^d \setminus \{ \mathbb O_d \} \\ (\beta_1,\ldots,\beta_k): \beta_i \in \mathbb N^m \setminus \{ \mathbb O_m \} \\ |\alpha_1|+|\beta_1|+\ldots+|\alpha_k|+|\beta_k| \leq D } }  \frac{1}{k!} \sum_{\sigma \in \mathfrak S_k} \mathbb E_{Q \sim \nu}\Big[ \prod_{1 \leq i \leq k} \Lambda_{\alpha_i,\beta_{\sigma(i)}}(Q) \Big]^2 \\
         =\ & 1+ \sum_{1 \leq k \leq D} \sum_{ \substack{ (\alpha_1,\ldots,\alpha_k): \alpha_i \in \mathbb N^d \setminus \{ \mathbb O_d \} \\ (\beta_1,\ldots,\beta_k): \beta_i \in \mathbb N^m \setminus \{ \mathbb O_m \} \\ |\alpha_1|+|\beta_1|+\ldots+|\alpha_k|+|\beta_k| \leq D } }  \mathbb E_{Q \sim \nu}\Big[ \prod_{1 \leq i \leq k} \Lambda_{\alpha_i,\beta_i}(Q) \Big]^2 \,,
    \end{align*}
    where in equality we use that fix any $\sigma \in \mathfrak S_k$ there is a one-to-one correspondence between $(\beta_1,\ldots,\beta_k)$ and $(\beta_{\sigma(1)},\ldots,\beta_{\sigma(k)})$. This completes our proof.
\end{proof}

\begin{remark}
    We can bound the right hand side of \eqref{eq-CorAdv-transform} regarding low-coordinate advantage in the similar manner. Observe that for any $\mathsf A \in (\mathbb N^d)^{\otimes n},\mathsf B \in (\mathbb N^m)^{\otimes n}$ with $\| \mathsf A \|_0+ \| \mathsf B \|_0 \leq D$, we can characterize $(\mathsf A,\mathsf B)$ with two subsets $S,K \subset [n], |S|,|K| \leq D$ and two ordered families of integer vectors $(\alpha_1,\ldots,\alpha_{|S|}), (\beta_1, \ldots, \beta_{|K|})$ such that $\alpha_i \in \mathbb N^d \setminus \{ \mathbb O_d \}, \beta_j \in \mathbb N^m \setminus \{ \mathbb O_m \}$ and 
    \begin{align*}
        \sum_{1\leq i \leq |S|} \|\alpha_i \|_0 + \sum_{1 \leq j \leq |K|} \|\beta_j \|_0 \leq D \,.
    \end{align*}
    Using similar arguments as in Lemma~\ref{lem-Adv-relax-1} we have that
    \begin{align}
        \mathsf{CorAdv}_{\leq D}(\Pb,\Qb) \leq 1 + \sum_{1 \leq k \leq D} \sum_{ \substack{ (\alpha_1,\ldots,\alpha_k): \alpha_i \in \mathbb N^d \setminus \{ \mathbb O_d \} \\ (\beta_1,\ldots,\beta_k): \beta_i \in \mathbb N^m \setminus \{ \mathbb O_m \} \\ \| \alpha_1 \|_0+ \| \beta_1 \|_0 +\ldots+\|\alpha_k\|_0+\|\beta_k\|_0 \leq D } }  \mathbb E_{Q \sim \nu}\Big[ \prod_{1 \leq i \leq k} \Lambda_{\alpha_i,\beta_i}(Q) \Big]^2 \,.  \label{eq-CorAdv-relax-1}
    \end{align}
\end{remark}

\subsection{Relating to the simpler model}{\label{subsec:relating-linear-model}}

Although in \eqref{eq-Adv-relax-1} we have derived an explicit upper bound for the low-degree advantage, directly calculating $\Lambda_{\alpha,\beta}(Q)$ for general $d,m \geq 1$ appears complicated. An exception is the single-variate case $m=1$, where an explicit calculation of $\Lambda_{\alpha,\beta}(Q)$ is feasible. We direct the reader to Section~\ref{subsec:proof-main-thm-lower-bound-1} of the appendix for an alternative proof for the case $m=1$ that illustrates some intuitions through direct calculation.

Here we focus on the general case where $m,d \geq 1$. We will take an alternative approach by first relating the the left hand side of \eqref{eq-Adv-bound-m-geq-2} to an relatively easier hypothesis testing problems.

\begin{defn}{\label{def-easier-hypothesis-testing}}
    Fix an integer $k$. Consider the following hypothesis testing problem, where our observation are two vectors $X=(X_1,\ldots,X_k)^{\top} \in \mathbb R^{k*d}, Y=(Y_1,\ldots,Y_k)^{\top} \in \mathbb R^{k*m}$ such that:
    \begin{itemize}
        \item under the null hypothesis $\overline{\mathcal H}_0$, we sample $X_i \sim \mathcal N(0,\mathbb I_d)$ and $Y_i \sim \mathcal N(0,\mathbb I_m)$ independently.
        \item under the alternative hypothesis $\overline{\mathcal H}_1$, we first sample $X_1,\ldots,X_k,Z_1,\ldots,Z_k \sim \mathcal N(0,\mathbb I_d)$ and $Q \sim\nu$ independently, and then define $Y= \frac{1}{\sqrt{1+\sigma^2}} (XQ+\sigma Z)$.
    \end{itemize}
    In addition, we define $\mathtt P_k$ to be the law of $(X,Y)$ under $\overline{\mathcal H}_1$ and let $\mathtt Q_k$ to be the law of $(X,Y)$ under $\overline{\mathcal H}_0$, respectively.
\end{defn}

Our next lemma shows that we can bound $\Lambda(\alpha,\beta)$ by the chi-square norm between $\mathtt P_k$ and $\mathtt Q_k$.

\begin{lemma}{\label{lem-bound-by-chi-2-norm}}
    For each $k \geq 1$, we have
    \begin{align*}
        \sum_{ \substack{ (\alpha_1,\ldots,\alpha_k): \alpha_i \in \mathbb N^d \setminus \{ \mathbb O_d \} \\ (\beta_1,\ldots,\beta_k): \beta_i \in \mathbb N^m \setminus \{ \mathbb O_m \} \\ |\alpha_1|+|\beta_1|+\ldots+|\alpha_k|+|\beta_k| \leq D } } \mathbb E_{Q \sim \nu}\Big[ \prod_{1 \leq i \leq k} \Lambda_{\alpha_i,\beta_i}(Q) \Big]^2 \leq \chi^2(\mathtt P_k \| \mathtt Q_k) -1 \,.
    \end{align*}
\end{lemma}
\begin{proof}
    For $\overline{\mathsf A}=(\alpha_i:i \in [k])$ and $\overline{\mathsf B}=(\beta_i:i\in [k])$ we have 
    \begin{align*}
        \phi_{\overline{\mathsf A},\overline{\mathsf B}} = \prod_{1\leq i \leq k} \mathcal H_{\alpha_i}(X_i) \prod_{1 \leq j \leq k} \mathcal H_{\beta_j}(Y_j)
    \end{align*}
    are standard orthogonal under $\mathtt Q_k$. In addition,
    \begin{align*}
        \mathbb E_{\mathtt P_k}\big[ \phi_{\overline{\mathsf A},\overline{\mathsf B}} \big] = \mathbb E_{Q \sim \nu}\Big[ \prod_{1 \leq i \leq k} \Lambda_{\alpha_i,\beta_{i}}(Q) \Big] \,.
    \end{align*}
    Thus, we have 
    \begin{align*}
        & \sum_{ \substack{ (\alpha_1,\ldots,\alpha_k): \alpha_i \in \mathbb N^d \setminus \{ \mathbb O_d \} \\ (\beta_1,\ldots,\beta_k): \beta_i \in \mathbb N^m \setminus \{ \mathbb O_m \} \\ |\alpha_1|+|\beta_1|+\ldots+|\alpha_k|+|\beta_k| \leq D } } \mathbb E_{Q \sim \nu}\Big[ \prod_{1 \leq i \leq k} \Lambda_{\alpha_i,\beta_i}(Q) \Big]^2 \leq \sum_{ |\overline{\mathsf A}|+|\overline{\mathsf B}| \leq D } \mathbb E_{\mathtt P_k}\big[ \phi_{\overline{\mathsf A},\overline{\mathsf B}} \big]^2 -1 = \mathsf{Adv}_{\leq D}( \mathtt P_k,\mathtt Q_k ) -1 \,.
    \end{align*}
    And the results follows immediately since by definition $\mathsf{Adv}_{\leq D}(\mathtt P_k,\mathtt Q_k) \leq \chi^2(\mathtt P_k \| \mathtt Q_k)$.
\end{proof}

Now, plugging Lemma~\ref{lem-bound-by-chi-2-norm} into \eqref{eq-Adv-relax-1} we get that
\begin{align}{\label{eq-Adv-relax-3}}
    \mathsf{Adv}_{\leq D}(\Pb,\Qb) \leq 1 + \sum_{1 \leq k \leq D} \Big( \chi^2(\mathtt P_k \| \mathtt Q_k) -1 \Big) \,.
\end{align}
Thus, we have essentially reduce the problem into bounding the chi-square divergence between $\mathtt P_k$ and $\mathtt Q_k$.

\begin{remark}
    Regarding \eqref{eq-CorAdv-relax-1}, we can similarly show that
    \begin{align*}
        & \sum_{ \substack{ (\alpha_1,\ldots,\alpha_k): \alpha_i \in \mathbb N^d \setminus \{ \mathbb O_d \} \\ (\beta_1,\ldots,\beta_k): \beta_i \in \mathbb N^m \setminus \{ \mathbb O_m \} \\ \| \alpha_1 \|_0+ \| \beta_1 \|_0 +\ldots+\|\alpha_k\|_0+\|\beta_k\|_0 \leq D } }  \mathbb E_{Q \sim \nu}\Big[ \prod_{1 \leq i \leq k} \Lambda_{\alpha_i,\beta_i}(Q) \Big]^2 \\
        \leq\ & \sum_{ \| \overline{\mathsf A} \|_0+\|\overline{\mathsf B}\|_0 \leq D } \mathbb E_{\mathtt P_k}\big[ \phi_{\overline{\mathsf A},\overline{\mathsf B}} \big]^2 -1 \leq \mathsf{CorAdv}(\mathtt P_k,\mathtt Q_k)-1 \leq \chi^2(\mathtt P_k \| \mathtt Q_k)-1 \,.
    \end{align*}
    So we also have
    \begin{align}{\label{eq-CorAdv-relax-3}}
        \mathsf{CorAdv}_{\leq D}(\Pb,\Qb) \leq 1 + \sum_{1 \leq k \leq D} \Big( \chi^2(\mathtt P_k \| \mathtt Q_k) -1 \Big) \,.
    \end{align}
    Thus, our later bound on the chi-square divergence between $\mathtt P_k$ and $\mathtt Q_k$ also leads to the bound on low-coordinate degree advantage. We note that the technique of bounding the the low-coordinate advantage by a number of simpler ``local'' chi-square divergences appears in a similar spirit in \cite{Kun25}.
\end{remark}

\subsection{Proof of Item~(1)}{\label{subsec:proof-main-thm-lower-bound-2}}

In this subsection we prove Item~(1) of Theorem~\ref{main-thm-lower-bound}. Based on \eqref{eq-Adv-relax-3}, it suffices to show the following lemma.
\begin{lemma}{\label{lem-Adv-bound-m-geq-2}}
    Suppose that $m=o(d),\sigma=0$ and $D^4=o(\frac{d}{m})$. Then for all $k \leq D$ we have
    \begin{align}{\label{eq-Adv-bound-m-geq-2}}
        \mathbb E_{ (X,Y) \sim \mathtt Q_k } \Big[ \Big( \tfrac{ \mathrm{d}\mathtt P_k }{ \mathrm{d}\mathtt Q_k } (X,Y) \Big)^2 \Big] = 1+o(D^{-1}) \,. 
    \end{align}
\end{lemma}
\begin{proof}
    We will prove Lemma~\ref{lem-Adv-bound-m-geq-2} by calculating $\mathtt L(X,Y)=\tfrac{ \mathrm{d}\mathtt P_k }{ \mathrm{d}\mathtt Q_k } (X,Y)$ directly. We divide our proof into two cases.
    
    \noindent{\bf Case 1}: $k \leq m$ and $k^4 = o(d/m)$. We first calculate $\mathtt L(X,Y)$ in this case. Note that given $Y=B$ and $XX^{\top}=A$, the conditional distribution of $X$ under both $\mathtt P_k$ and $\mathtt Q_k$ is the uniform distribution over 
    \begin{align*}
        \Big\{ X \in \mathbb R^{k*d}: XX^{\top}=A \Big\} \,.
    \end{align*}
    Thus, we have
    \begin{align*}
        \mathtt L(X,Y) = \mathtt L(XX^{\top},Y) \,.
    \end{align*}
    Note that $XX^{\top}$ is a $k*k$ Wishart matrix with dimension $d$, it is well-known that the density of $XX^{\top}$ is given by (see, e.g., \cite{Wishart28} and \cite{BDER16})
    \begin{align}{\label{eq-density-XX^*}}
        f_{XX^{\top}}(A) = \mathbf 1_{ \{A \succeq 0\} } \cdot \omega(d,k) \operatorname{det}(A)^{\frac{1}{2}(d-k-1)} e^{-\frac{1}{2}\operatorname{tr}(A)} \,, 
    \end{align}
    where $\omega(\cdot,\cdot)$ is the Wishart constant defined by
    \begin{align}
        \frac{1}{\omega(s,t)} = \pi^{t(t-1)/4} 2^{st/2} \prod_{j=1}^{t} \Gamma\big( \tfrac{s-j+1}{2} \big) \,. \label{eq-def-omega(s,t)}
    \end{align}
    In addition, we claim that the conditional distribution of $Y$ given $\{ XX^{\top}=A \}$ is given by
    \begin{align}
        Y \mid \{ XX^{\top}=A \} \overset{d}{=} A^{\frac{1}{2}} \cdot \mathbf Z_{k*m} \,,  \label{eq-conditional-distribution-Y}
    \end{align}
    where $\mathbf Z_{k*m}$ is the upper-left $k*m$ submatrix of a uniform $d*d$ orthogonal matrix $\mathbf \Gamma_d$. In fact, conditioned on the spectral decomposition of $X$ 
    \begin{align*}
        X= \begin{pmatrix} A^{\frac{1}{2}} & \mathbb O_{k*(d-k)} \end{pmatrix} U \mbox{ where } U \in O(d) \,,
    \end{align*}
    we have
    \begin{align*}
        Y = XQ = \begin{pmatrix} A^{\frac{1}{2}} & \mathbb O_{k*(d-k)} \end{pmatrix} UQ
    \end{align*}
    and thus \eqref{eq-conditional-distribution-Y} follows from $UQ$ is uniformly distributed over the Grassmanian manifold (whose first $k$ row has the same distribution as $\mathbf Z_{k*m}$). Applying Lemma~\ref{lem-denisty-submatrix-uniform-O(n)}, we see that the conditional density of $Y$ given $\{ XX^{\top}=A \}$ is 
    \begin{align*}
        f(Y\mid XX^{\top}=A)= \frac{ \omega(d-m,k) }{ (2\pi)^{km/2} \omega(d,k) } \operatorname{det}(\mathbb I_m-Y^{\top} A^{-1} Y)^{(d-k-m-1)/2} \operatorname{det}(A)^{-m/2} \cdot \zeta( A^{-1/2} Y ) \,,
    \end{align*}
    where $\zeta(A)$ is the indicator function that all the eigenvalues of $A^{\top} A$ lies in $[0,1]$. In addition, note that under $\mathtt Q_k$, the density of $Y$ is independent of $X$ and is given by
    \begin{align}{\label{eq-density-Y-mathtt-Q}}
        \frac{1}{(2\pi)^{km/2}} e^{ -\operatorname{tr}(Y^{\top}Y)/2 } \,.
    \end{align}
    Thus, we have
    \begin{align}
        \mathtt L(A,Y) = \tfrac{ \omega(d-m,k) }{ \omega(d,k) } \cdot e^{ \frac{\operatorname{tr}(Y^{\top}Y)}{2} } \operatorname{det}(\mathbb I_m-Y^{\top} A^{-1} Y)^{\frac{d-k-m-1}{2}} \operatorname{det}(A)^{-m/2} \cdot \zeta( A^{-1/2} Y ) \,. \label{eq-exact-form-mathtt-L}
    \end{align}
    Given the explicit expression of $\mathtt L$, we can now calculate the chi-square divergence directly.
    Plugging \eqref{eq-exact-form-mathtt-L} into the formula 
    \begin{align}{\label{eq-second-moment-formula}}
        \mathbb E_{\mathtt Q_k}\big[ \mathtt L(XX^{\top},Y)^2 \big] = \int_{A \succeq 0} \int_Y \mathtt L(A,Y)^2 \cdot \frac{ e^{-\operatorname{tr}(Y^{\top}Y)/2} }{ (2\pi)^{km/2} } \cdot \omega(d,k) \operatorname{det}(A)^{\frac{1}{2}(d-k-1)} e^{-\frac{1}{2}\operatorname{tr}(A)} \mathrm{d}A\mathrm{d}Y \,,
    \end{align}
    we obtain that $\mathbb E_{\mathtt Q_k}\big[ \mathtt L(X,Y)^2 \big]$ equals
    \begin{align}
        \int_{A\succeq 0} \int_Y \frac{ \omega(d-m,k)^2 \operatorname{det}(\mathbb I_m-Y^{\top} A^{-1} Y)^{d-k-m-1} \operatorname{det}(A)^{\frac{1}{2}(d-k-2m-1)} \zeta(A^{-1/2} Y) }{ \omega(d,k) (2\pi)^{\frac{km}{2}} e^{ \frac{1}{2}(\operatorname{tr}(A)-\operatorname{tr}(Y^{\top}Y)) } } \mathrm{d}A\mathrm{d}Y \,.  \label{eq-second-moment-mathtt-Q-relax-1}
    \end{align}
    By substituting $Y$ with $A^{1/2}Z$, we get that
    \begin{align}
        \eqref{eq-second-moment-mathtt-Q-relax-1} = \int_{A\succeq 0} \int_Z \frac{ \omega(d-m,k)^2 \operatorname{det}(\mathbb I_m-Z^{\top}Z)^{d-k-m-1} \operatorname{det}(A)^{\frac{1}{2}(d-k-m-1)} \zeta(Z) }{ \omega(d,k) (2\pi)^{\frac{km}{2}} e^{ \frac{1}{2}(\operatorname{tr}(A)-\operatorname{tr}(Z^{\top}AZ)) } } \mathrm{d}A\mathrm{d}Z \,. \label{eq-second-moment-mathtt-Q-relax-2}
    \end{align}
    Note that
    \begin{align*}
        \operatorname{tr}(A)-\operatorname{tr}(Z^{\top}AZ) = \operatorname{tr}(A(\mathbb I_k-ZZ^{\top})) \,,
    \end{align*}
    by substituting $A$ with $(\mathbb I_k-ZZ^{\top})^{-1/2}B(\mathbb I_k-ZZ^{\top})^{-1/2}$, we get that (note that $\operatorname{det}(\mathbb I_m-Z^{\top}Z)=\operatorname{det}(\mathbb I_k-ZZ^{\top})$)
    \begin{align}
        \eqref{eq-second-moment-mathtt-Q-relax-2} = \int_{B\succeq 0} \int_Z \frac{ \omega(d-m,k)^2 \operatorname{det}(\mathbb I_m-Z^{\top}Z)^{\frac{1}{2}(d-3k-m-1)} \operatorname{det}(B)^{\frac{1}{2}(d-k-m-1)} \zeta(Z) }{ \omega(d,k) (2\pi)^{\frac{km}{2}} e^{\frac{1}{2}\operatorname{tr}(B)} } \mathrm{d}B\mathrm{d}Z \,. \label{eq-second-moment-mathtt-Q-relax-3}
    \end{align}
    Note that
    \begin{align*}
        & \int_{B\succ 0} e^{-\frac{1}{2}\operatorname{tr}(B)} \operatorname{det}(B)^{\frac{1}{2}(d-k-m-1)} \mathrm{d}B = \frac{1}{\omega(d-m,k)} \,, \\
        & \int_{Z} \operatorname{det}(\mathbb I_m-Z^{\top}Z)^{\frac{1}{2}(d-3k-m-1)} \chi(Z) \mathrm{d}Z = \frac{ (2\pi)^{mk/2} \omega(d-2k,k) }{\omega(d-2k-m,k)} \,.
    \end{align*}
    We then get that
    \begin{align}
        \eqref{eq-second-moment-mathtt-Q-relax-3} &= \frac{ (2\pi)^{mk/2} \omega(d-2k,k) }{\omega(d-2k-m,k)} \cdot \frac{1}{\omega(d-m,k)} \cdot \frac{ \omega(d-m,k)^2 }{ \omega(d,k) (2\pi)^{km/2} } \nonumber \\
        &= \frac{ \omega(d-m,k) }{ \omega(d-m-2k,k) } \cdot \frac{ \omega(d-2k,k) }{ \omega(d,k) } \,. \label{eq-second-moment-mathtt-Q-relax-4}
    \end{align}
    Applying Lemma~\ref{lem-bound-omega-ratio}, we get that (recall the assumption $k^4 \leq D^4=o(d/m)$)
    \begin{equation*}
        \eqref{eq-second-moment-mathtt-Q-relax-4} = [1+o(1)] \cdot \big( \tfrac{d}{d-m} \big)^{k^2} = [1+o(1)] \cdot \big( 1+\tfrac{mk^2}{d} \big) = 1+o(D^{-1}) \,. 
    \end{equation*}
    
    \noindent{\bf Case 2}: $m^4 \leq k^4=o(d/m)$. Similar as in case~1, we first calculate $\mathtt L(X,Y)$ explicitly. In this case, again we have
    \begin{align*}
        \mathtt L(X,Y) = \mathtt L(XX^{\top},Y)
    \end{align*}
    and the density of $XX^{\top}$ is given by \eqref{eq-density-XX^*}. Using Lemma~\ref{lem-denisty-submatrix-uniform-O(n)}, we see that the conditional density of $Y$ given $\{ XX^{\top}=A \}$ is 
    \begin{align*}
        f(Y\mid XX^{\top}=A)= \frac{ \omega(d-k,m) }{ (2\pi)^{km/2} \omega(d,m) } \operatorname{det}( \mathbb I_k - A^{-\frac{1}{2}} YY^{\top} A^{-\frac{1}{2}} )^{\frac{d-k-m-1}{2}} \operatorname{det}(A)^{-\frac{m}{2}} \cdot \zeta( YA^{-1/2} ) \,.
    \end{align*}
    Since the density of $Y$ under $\mathtt Q_k$ is given by \eqref{eq-density-Y-mathtt-Q}, we get that
    \begin{align}
        \mathtt L(A,Y) = \tfrac{ \omega(d-k,m) }{ \omega(d,m) } \cdot e^{ \frac{\operatorname{tr}(Y^{\top}Y)}{2} } \operatorname{det}( \mathbb I_k - A^{-\frac{1}{2}} YY^{\top} A^{-\frac{1}{2}} )^{\frac{d-k-m-1}{2}} \operatorname{det}(A)^{-\frac{m}{2}} \cdot \zeta( YA^{-1/2} ) \,.  \label{eq-exact-form-mathtt-L-case-2}
    \end{align}
    Now we calculate the chi-square divergence using the explicit expression of $\mathtt L$.
    Plugging \eqref{eq-exact-form-mathtt-L-case-2} into \eqref{eq-second-moment-formula}, we obtain that $\mathbb E_{\mathtt Q_k}\big[ \mathtt L(X,Y)^2 \big]$ equals
    \begin{align}
        \int_{A\succ 0} \int_Y \frac{ \omega(d-k,m)^2 \omega(d,k) \operatorname{det}(\mathbb I_k-A^{-\frac{1}{2}} YY^{\top} A^{-\frac{1}{2}})^{d-k-m-1} \operatorname{det}(A)^{\frac{1}{2}(d-k-2m-1)} \zeta( YA^{-1/2} ) }{ \omega(d,m)^2 (2\pi)^{\frac{km}{2}} e^{ \frac{1}{2}(\operatorname{tr}(A)-\operatorname{tr}(Y^{\top}Y)) } } \mathrm{d}A\mathrm{d}Y \,.  \label{eq-second-moment-mathtt-Q-relax-5}
    \end{align}
    Substituting $Y=A^{1/2} Z$, we get that
    \begin{align}
        \eqref{eq-second-moment-mathtt-Q-relax-5} = \int_{A\succeq 0} \int_Z \frac{ \omega(d-k,m)^2 \omega(d,k) \operatorname{det}(\mathbb I_k-ZZ^{\top})^{d-k-m-1} \operatorname{det}(A)^{\frac{1}{2}(d-k-m-1)} \zeta(Z) }{ \omega(d,m)^2 (2\pi)^{\frac{km}{2}} e^{ \frac{1}{2}(\operatorname{tr}(A)-\operatorname{tr}(A ZZ^{\top})) } } \mathrm{d}A\mathrm{d}Z \,. \label{eq-second-moment-mathtt-Q-relax-6}
    \end{align}
    Substituting $A=(\mathbb I_k-ZZ^{\top})^{-1/2} B (\mathbb I_k-ZZ^{\top})^{-1/2}$, we get that
    \begin{align}
        \eqref{eq-second-moment-mathtt-Q-relax-6} = \int_{B\succ 0} \int_Z \frac{ \omega(d-k,m)^2 \omega(d,k) \operatorname{det}(\mathbb I_k-ZZ^{\top})^{d-3k-m-1} \operatorname{det}(B)^{\frac{1}{2}(d-k-m-1)} \zeta(Z) }{ \omega(d,m)^2 (2\pi)^{\frac{km}{2}} e^{ \frac{1}{2} \operatorname{tr}(B) } } \mathrm{d}B \mathrm{d}Z \,. \label{eq-second-moment-mathtt-Q-relax-7}
    \end{align}
    Note that
    \begin{align*}
        & \int_{B\succeq 0} e^{-\frac{1}{2}\operatorname{tr}(B)} \operatorname{det}(B)^{\frac{1}{2}(d-k-m-1)} \mathrm{d}B = \frac{1}{\omega(d-m,k)} \,, \\
        & \int_{Z} \operatorname{det}(\mathbb I_k-Z^{\top}Z)^{\frac{1}{2}(d-3k-m-1)} \zeta(Z) \mathrm{d}Z = \frac{ (2\pi)^{mk/2} \omega(d-2k,m) }{\omega(d-3k,m)} \,.
    \end{align*}
    We then get that
    \begin{align*}
        \eqref{eq-second-moment-mathtt-Q-relax-7} &= \frac{\omega(d-k,m)^2}{\omega(d,m)^2} \cdot \frac{\omega(d,k)}{\omega(d-m,k)} \cdot \frac{\omega(d-2k,m)}{\omega(d-3k,m)} \\
        &= [1+o(1)] \cdot d^{mk} \cdot (d-m)^{-mk/2} \cdot (d-3k)^{-mk/2} \\
        &= [1+o(1)] \cdot \big( \tfrac{d^2}{(d-m)(d-3k)} \big)^{mk/2} = 1+o(mk^2/d) = 1+o(D^{-1}) \,,
    \end{align*}
    where in the second equality we use Lemma~\ref{lem-bound-omega-ratio} and in the last inequality we use the assumption $m^4 \leq k^4 =o(d/m)$. Combining the two cases leads to the desired result.
\end{proof}

\subsection{Proof of Item~(2)}{\label{subsec:proof-main-thm-lower-bound-3}}

In this subsection we prove Item~(2) of Theorem~\ref{main-thm-lower-bound}. Again, using \eqref{eq-Adv-relax-3} it suffices to show the following lemma.

\begin{lemma}{\label{lem-bound-by-chi-2-norm-m=d}}
    Suppose $m=d,\sigma=\omega(1)$ and $D=o(\sigma)$. Then for all $k \leq D$ we have
    \begin{align}{\label{eq-bound-by-chi-2-norm-m=d}}
        \mathbb E_{ (X,Y) \sim \mathtt Q_k } \Big[ \Big( \tfrac{ \mathrm{d}\mathtt P_k }{ \mathrm{d}\mathtt Q_k } (X,Y) \Big)^2 \Big] = 1+o(D^{-1}) \,. 
    \end{align}
\end{lemma}
\begin{proof}
We will prove Lemma~\ref{lem-bound-by-chi-2-norm-m=d} by calculating $\mathtt L(X,Y)=\tfrac{ \mathrm{d}\mathtt P_k }{ \mathrm{d}\mathtt Q_k } (X,Y)$ directly. Clearly, the probability density function of $(X,Y)$ under $\mathtt Q_k$ is given by
\begin{align*}
    f_{\mathtt Q_k}(X,Y) = \frac{ e^{-\| X \|_{ \operatorname{F} }^2/2} }{ (2\pi)^{kd/2} } \cdot \frac{ e^{-\| Y \|_{ \operatorname{F} }^2/2} }{ (2\pi)^{kd/2} } \,.
\end{align*}
In addition, the probability density function of $(X,Y)$ under $\mathtt P_k$ is given by
\begin{align*}
    f_{\mathtt P_k}(X,Y) = \frac{ e^{-\| X \|_{ \operatorname{F} }^2/2} }{ (2\pi)^{kd/2} } \cdot \int_{O(d)} \frac{ e^{ - \frac{1+\sigma^2}{2\sigma^2} \| Y-\frac{XQ}{\sqrt{1+\sigma^2}} \|_{\Fop}^2 } }{ (2\pi\sigma^2/(1+\sigma^2))^{kd} } \mathrm{d}Q \,.
\end{align*}
Thus, we get that
\begin{align*}
    \mathtt L(X,Y) = \big( \tfrac{1+\sigma^2}{\sigma^2} \big)^{kd/2} e^{ -(\|X\|_{\Fop}^2+\|Y\|_{\Fop}^2) /2\sigma^2 } \cdot \int_{O(d)} e^{ \sqrt{1+\sigma^2} \langle X,YQ \rangle/\sigma^2 } \mathrm{d}Q \,.
\end{align*}
Now we calculate the chi-square divergence using the explicit expression of $\mathtt L$. Using replica's trick, we have
\begin{align*}
    \mathtt L(X,Y)^2 = \big( \tfrac{1+\sigma^2}{\sigma^2} \big)^{kd} e^{ -(\|X\|_{\Fop}^2+\|Y\|_{\Fop}^2)/\sigma^2 } \cdot \int_{O(d) \times O(d)} e^{ \sqrt{1+\sigma^2} \langle X,Y(Q_1+Q_2) \rangle/\sigma^2 } \mathrm{d}Q_1 \mathrm{d}Q_2 \,.
\end{align*}
We then have
\begin{align}
    & \mathbb E_{(X,Y) \sim \mathtt Q_k}\Big[ \mathtt L(X,Y)^2 \Big] \nonumber \\
    =\ & \big( \tfrac{1+\sigma^2}{\sigma^2} \big)^{kd} \int_{O(d) \times O(d)} \mathbb E_{(X,Y) \sim \mathtt Q_k} \Big[ e^{ -(\|X\|_{\Fop}^2+\|Y\|_{\Fop}^2)/\sigma^2 } \cdot e^{ \sqrt{1+\sigma^2} \langle X,Y(Q_1+Q_2) \rangle/\sigma^2 } \Big] \mathrm{d}Q_1 \mathrm{d}Q_2 \nonumber \\
    =\ & \big( \tfrac{1+\sigma^2}{\sigma^2} \big)^{kd} \big( \tfrac{\sigma^2}{2+\sigma^2} \big)^{kd/2} \int_{O(d) \times O(d)} \mathbb E_{Y \sim \mathtt Q_k}\Big[  e^{ -\|Y\|_{\Fop}^2/\sigma^2 } \cdot e^{ (1+\sigma^2) \|Y(Q_1+Q_2)\|_{\Fop}^2/2\sigma^2(2+\sigma^2) } \Big] \mathrm{d}Q_1 \mathrm{d}Q_2 \nonumber \\
    =\ & \big( \tfrac{1+\sigma^2}{\sigma^2} \big)^{kd} \big( \tfrac{\sigma^2}{2+\sigma^2} \big)^{kd/2} \int_{O(d)} \mathbb E_{Y \sim \mathtt Q_k}\Big[  e^{ -\|Y\|_{\Fop}^2/\sigma^2 } \cdot e^{ (1+\sigma^2)( \|Y\|_{\Fop}^2+\langle Y,YQ \rangle)/\sigma^2(2+\sigma^2) } \Big] \mathrm{d}Q \nonumber \\
    =\ & \big( \tfrac{1+\sigma^2}{\sigma^2} \big)^{kd} \big( \tfrac{\sigma^2}{2+\sigma^2} \big)^{kd/2} \int_{O(d)} \mathbb E_{Y \sim \mathtt Q_k}\Big[  e^{ -\|Y\|_{\Fop}^2/\sigma^2(\sigma^2+2) } \cdot e^{ (1+\sigma^2)\langle Y,YQ \rangle/\sigma^2(2+\sigma^2) } \Big] \mathrm{d}Q \,, \label{eq-chi-2-norm-d=m-relax-1}
\end{align}
where the second equality follows from Lemma~\ref{lem-Gaussian-exponential-moment} and the third equality follows from the fact that $Q_1 Q_2^{\top} \sim \nu$ for $Q_1 \perp Q_2$ and $Q_1,Q_2 \sim \nu$. Using Lemma~\ref{lem-Gaussian-quadratic-form-moment} with 
\begin{align*}
    A=A(Q) = \frac{1}{\sigma^2(\sigma^2+2)} \mathbb I_d - \frac{\sigma^2+1}{2\sigma^2(\sigma^2+2)} (Q+Q^{\top}) \,,
\end{align*}
we get that
\begin{align}
    \eqref{eq-chi-2-norm-d=m-relax-1} = \big( \tfrac{1+\sigma^2}{\sigma^2} \big)^{kd} \big( \tfrac{\sigma^2}{2+\sigma^2} \big)^{kd/2} \int_{O(d)} \mathrm{det}\big( \mathbb I_d + 2A(Q) \big)^{-k/2} \mathrm{d}Q \,. \label{eq-chi-2-norm-d=m-relax-2}
\end{align}
Note that (denote $\epsilon=\epsilon(\sigma)=(1+\sigma^2)^{-1}=o(1)$)
\begin{align*}
    \mathrm{det}\big( \mathbb I_d + 2A(Q) \big) &= \mathrm{det}\Big( \tfrac{1+(\sigma^2+1)^2}{\sigma^2(\sigma^2+2)} \mathbb I_d - \tfrac{\sigma^2+1}{\sigma^2(\sigma^2+2)} (Q+Q^{\top}) \Big) \\
    &= \Big( \tfrac{ (\sigma^2+1)^2 }{ \sigma^2(\sigma^2+2) } \Big)^d \cdot \mathrm{det}\Big( (1+\epsilon^2) \mathbb I_d - \epsilon(Q+Q^{\top}) \Big) \\
    &= \Big( \tfrac{ (\sigma^2+1)^2 }{ \sigma^2(\sigma^2+2) } \Big)^d \cdot \mathrm{det}\Big( (\mathbb I_d-\epsilon Q) (\mathbb I_d-\epsilon Q^{\top}) \Big) = \tfrac{ (\sigma^2+1)^2 }{ \sigma^2(\sigma^2+2) } \cdot \mathrm{det}\big( \mathbb I_d-\epsilon Q \big)^2 \,.
\end{align*}
Thus, we have 
\begin{align}
    \eqref{eq-chi-2-norm-d=m-relax-2} = \int_{O(d)} \mathrm{det}\big( \mathbb I_d-\tfrac{1}{1+\sigma^2} Q \big)^{-k} \mathrm{d} Q = 1 + O\big( \tfrac{k}{1+\sigma^2} \big) = 1+ o(D^{-1}) \,,
\end{align}
where the second equality follows from Lemma~\ref{lem-bound-uniform-orthogonal-integration} and the last equality follows from $k \leq D$ and $D=o(\sigma)$. 
\end{proof}

\subsection{Proof of Item~(3)}{\label{subsec:proof-main-thm-upper-bound}}

In this subsection we prove Item~(3) in Theorem~\ref{main-thm-lower-bound}. Denote 
\begin{equation}{\label{eq-def-detect-polynomial}}
    f(X,Y) = \Big( \| Y \|_{\Fop}^2 - \| X \|_{\Fop}^2 \Big)^2 \,. 
\end{equation}
It suffices to show the following estimation.
\begin{lemma}{\label{lem-moment-estimation}}
    Assuming $d=m$ and $\sigma=o(1)$, then we have
    \begin{enumerate}
        \item[(1)] $\mathbb E_{\Pb}[f]=o(nd)$ and $\mathbb E_{\Qb}[f]=\Theta(nd)$.
        \item[(2)] $\operatorname{Var}_{\Pb}[f], \operatorname{Var}_{\Qb}[f] = o(n^2 d^2)$.
    \end{enumerate}
\end{lemma}
\begin{proof}
    Note that under $\Pb$ we have $Y=(XQ+\sigma Z)/\sqrt{1+\sigma^2}$, thus
    \begin{align*}
        \| Y \|_{\Fop}^2 - \| X \|_{\Fop}^2 = \tfrac{1}{1+\sigma^2} \| XQ+\sigma Z \|_{\Fop}^2 - \| X \|_{\Fop}^2 = \tfrac{\sigma^2}{1+\sigma^2} \big( \| Z \|_{\Fop}^2 - \| X \|_{\Fop}^2 \big) + \tfrac{2\sigma}{1+\sigma^2} \langle X,Z \rangle \,.
    \end{align*}
    Thus, we have
    \begin{align}
        \mathbb E_{\Pb}[f^2] = \mathbb E_{\Pb}\Big[ \big( \| Y \|_{\Fop}^2 - \| X \|_{\Fop}^2 \big)^4 \Big] &\leq 2^4 \cdot \Big( \mathbb E_{\Pb}\Big[ \sigma^4 \big( \| Z \|_{\Fop}^2 - \| X \|_{\Fop}^2 \big)^4 \Big] + \mathbb E_{\Pb}\Big[ \sigma^2 \langle X,Z \rangle^4 \Big] \Big) \nonumber \\
        &\leq O(1) \cdot \sigma^2 \cdot n^2 d^2 = o(n^2d^2) \,, \label{eq-var-Pb}
    \end{align}
    where the second inequality follows from $\| Z \|_{\Fop}^2 - \| X \|_{\Fop}^{2}$ and $\langle X,Z \rangle$ are sums of $nd$ independent random variables with finite fourth moment. Thus, using Cauchy-Schwartz inequality we have 
    \begin{equation}{\label{eq-mean-Pb}}
        |\mathbb E_{\Pb}[f]| \leq \mathbb E_{\Pb}[f^2]^{\frac{1}{2}} = o(nd) \,.
    \end{equation}
    In addition, under $\Qb$ we have $X,Y$ are independent $n*d$ Gaussian random matrices with i.i.d.\ $\mathcal N(0,1)$ entries. Thus, we have
    \begin{align}
        & \mathbb E_{\Qb}\big[ f \big] = \mathbb E_{\Qb}\Big[ \Big( \sum_{1 \leq i \leq n} \sum_{1 \leq j \leq d} (Y_{i,j}^2-X_{i,j}^2) \Big)^2 \Big] = 4nd \,, \label{eq-expectation-Qb} \\
        & \mathbb E_{\Qb}\big[ f^2 \big] = \mathbb E_{\Qb}\Big[ \Big( \sum_{1 \leq i \leq n} \sum_{1 \leq j \leq d} (Y_{i,j}^2-X_{i,j}^2) \Big)^4 \Big] = 16n^2d^2 + O(nd) \,. \label{eq-var-Qb}
    \end{align}
    Combining \eqref{eq-var-Pb}, \eqref{eq-mean-Pb}, \eqref{eq-expectation-Qb} and \eqref{eq-var-Qb} yields the desired result.
\end{proof}

{\bf Acknowledgment.} We thank Jian Ding for helpful discussions, Shuyang Gong and Jiaming Xu for helpful discussions on the forthcoming manuscript \cite{GWX25+, NSX26+}, and Hang Du for pointing out the reference \cite{MM13}. We also warmly thank four anonymous reviewers for their careful reading and helpful comments which lead to a significant improvement on exposition. This work is partially supported by National Key R$\&$D program of China (Project No. 2023YFA1010103) and NSFC Key Program (Project No. 12231002).

\appendix

\section{Proof of Theorem~\ref{main-thm-lower-bound}, Item~(1) in single-variate case}{\label{subsec:proof-main-thm-lower-bound-1}}

In this section, we provide an alternative proof of Item~(1) of Theorem~\ref{main-thm-lower-bound} in the single-variate case (with a milder assumption on $D$)
\begin{align*}
    m=1, \sigma=0, D^6=o(d) \,.
\end{align*}
While mathematically our Theorem~\ref{main-thm-lower-bound} is self-contained from this appendix, we believe that the analysis of the single-variate case serves as a helpful warm-up and illustrate some intuitions on how we obtain the general results in Theorem~\ref{main-thm-lower-bound}.

Based on Lemma~\ref{lem-Adv-relax-1}, we see that
\begin{align}
    \mathsf{Adv}_{\leq D}(\Pb,\Qb)^2 \leq 1+ \sum_{1 \leq k \leq D} \sum_{ \substack{ \alpha_1, \ldots, \alpha_k \in \mathbb N^d \\ \beta_1, \ldots, \beta_k \in \mathbb N \\ 0 < |\alpha_i|,|\beta_j| \leq D } } \mathbb E_{Q \sim \nu}\Big[ \prod_{1 \leq i \leq k} \Lambda_{\alpha_i,\beta_i}(Q) \Big]^2 \,, \label{eq-Adv-bound-m=1}
\end{align}
where 
\begin{align*}
    \Lambda_{\alpha,\beta}(Q) = \mathbb E_{U \sim \mathcal N(0,\mathbb I_d)}\Big[ \mathcal H_{\alpha}(U) \mathcal H_{\beta}(\langle U,Q \rangle) \Big] \,.
\end{align*}
In this case, we can calculate $\Lambda_{\alpha,\beta}(Q)$ for each $\alpha\in\mathbb N^d$ and $\beta\in\mathbb N$ directly, as incorporated in the next lemma.
\begin{lemma}{\label{lem-bound-Lambda-alpha,beta}}
    We have
    \begin{equation}{\label{eq-bound-Lambda-alpha,beta}}
        \Lambda_{\alpha,\beta}(Q) = \mathbf 1_{ \{ |\alpha|=\beta \} } \cdot \mathfrak{M}(\alpha;\beta) Q^{\alpha} \,,
    \end{equation}
    where 
    \begin{equation}{\label{eq-def-mathfrak-M}}
        \mathfrak{M}(\alpha;\beta) = \sqrt{ \frac{\alpha!}{\beta!} } \cdot \binom{\beta}{\alpha} \,.
    \end{equation}
\end{lemma}
\begin{proof}
    By direct calculation, we have that
    \begin{align*}
        \Lambda_{\alpha,\beta}(Q) & \overset{\text{Lemma~\ref{lem-expan-inner-product-in-Hermite}}}{=} \sum_{ \gamma\in\mathbb N^d: |\gamma|=\beta } \mathfrak{M}(\gamma;\beta) \mathbb E_{U \sim \mathcal N(0,\mathbb I_d)}\Big[ Q^{\gamma} \cdot \mathcal H_{\gamma}(U) \mathcal H_{\alpha}(U) \Big] \\
        & = \sum_{ \gamma\in\mathbb N^d: |\gamma|=\beta } \mathbf 1_{ \{ \gamma=\alpha \} } \cdot \mathfrak{M}(\gamma;\beta) Q^{\gamma} = \mathbf 1_{ \{ |\alpha|=\beta \} } \cdot \mathfrak{M}(\alpha;\beta) Q^{\alpha} \,,
    \end{align*}
    as desired.
\end{proof}

We can now finish the proof of Theorem~\ref{main-thm-lower-bound}, Item~(1).
\begin{proof}[Proof of Theorem~\ref{main-thm-lower-bound}, Item~(1)]
    Using Lemma~\ref{lem-bound-Lambda-alpha,beta}, we get that
    \begin{align}
        & \sum_{ \substack{ \alpha_1, \ldots, \alpha_k \in \mathbb N^d; \beta_1, \ldots, \beta_k \in \mathbb N \\ 0 < |\alpha_i|,|\beta_j| \leq D } } \mathbb E_{Q \sim \nu}\Big[ \prod_{1 \leq i \leq k} \Lambda_{\alpha_i,\beta_i}(Q) \Big]^2 \nonumber \\
        =\ & \sum_{ \substack{ \alpha_1, \ldots, \alpha_k \in \mathbb N^d; \beta_1, \ldots, \beta_k \in \mathbb N \\ 0 < |\alpha_i|,|\beta_j| \leq D } } \mathbb E_{Q \sim \nu}\Big[ \prod_{1 \leq i \leq k} \big( \mathbf 1_{ \{ |\alpha_i|=\beta_i \} } \cdot \mathfrak{M}(\alpha_i;\beta_i) Q^{\alpha_i} \big) \Big]^2 \nonumber \\
        =\ & \sum_{ \substack{ \alpha_1, \ldots, \alpha_k \in \mathbb N^d \\ \alpha_i \neq \mathbb O_d, |\alpha_i| \leq D } } \mathfrak{M}(\alpha_1;|\alpha_1|)^2 \ldots \mathfrak{M}(\alpha_k;|\alpha_k|)^2 \mathbb E_{Q \sim \nu}\big[ Q^{\alpha_1+\ldots+\alpha_k} \big]^2 \nonumber \\
        \overset{\eqref{eq-def-mathfrak-M}}{\leq} & \sum_{ \gamma \neq \mathbb O_d,|\gamma|\leq D } \mathbb E_{Q \sim \nu}\big[ Q^{\gamma} \big]^2 \sum_{ \substack{ \alpha_1+ \ldots +\alpha_k = \gamma \\ \alpha_i \neq 0, |\alpha_i| \leq D } } \tbinom{|\alpha_1|}{\alpha_1} \ldots \tbinom{|\alpha_k|}{\alpha_k} \,. \label{eq-m=1-relax-1}
    \end{align}
    Using Lemma~\ref{lem-moments-uniform-sphere}, we see that 
    \begin{align}
        \eqref{eq-m=1-relax-1} &= \sum_{ \gamma \neq \mathbb O_d,|\gamma|\leq D/2 } \mathbb E_{Q \sim \nu}\big[ Q^{2\gamma} \big]^2 \sum_{ \substack{ \alpha_1+ \ldots +\alpha_k = 2\gamma \\ \alpha_i \neq \mathbb O_d, |\alpha_i| \leq D } } \tbinom{|\alpha_1|}{\alpha_1} \ldots \tbinom{|\alpha_k|}{\alpha_k} \nonumber \\
        &= \sum_{ \gamma \neq \mathbb O_d,|\gamma|\leq D/2 } \Big( \tfrac{ \Gamma(\frac{d}{2}) (2\gamma_1-1)!! \ldots (2\gamma_d-1)!! }{ \Gamma(\frac{d}{2}+\gamma_1+\ldots+\gamma_d) 2^{\gamma_1+\ldots+\gamma_d} } \Big)^2 \sum_{ \substack{ \alpha_1+ \ldots +\alpha_k = 2\gamma \\ \alpha_i \neq \mathbb O_d, |\alpha_i| \leq D } } \tbinom{|\alpha_1|}{\alpha_1} \ldots \tbinom{|\alpha_k|}{\alpha_k} \nonumber \\
        &\leq \sum_{ \gamma \neq \mathbb O_d,|\gamma|\leq D/2 } d^{-2|\gamma|} D^{2|\gamma|} \sum_{ \substack{ \alpha_1+ \ldots +\alpha_k = 2\gamma \\ \alpha_i \neq \mathbb O_d, |\alpha_i| \leq D } } \tbinom{|\alpha_1|}{\alpha_1} \ldots \tbinom{|\alpha_k|}{\alpha_k} \,, \label{eq-m=1-relax-2}
    \end{align}
    where the inequality follows from 
    \begin{align*}
        \Big( (2\gamma_1-1)!! \ldots (2\gamma_d-1)!! \Big)^2 \leq D^{2(\gamma_1+\ldots+\gamma_d)} = D^{2|\gamma|}
    \end{align*}
    for $|\gamma| \leq D/2$. In addition, we have that 
    \begin{align*}
        \sum_{ \substack{ \alpha_1+ \ldots +\alpha_k = 2\gamma \\ \alpha_i \neq \mathbb O_d, |\alpha_i| \leq D } } \tbinom{|\alpha_1|}{\alpha_1} \ldots \tbinom{|\alpha_k|}{\alpha_k} &\leq \sum_{ \substack{ \alpha_1+ \ldots +\alpha_k = 2\gamma \\ \alpha_i \neq \mathbb O_d, |\alpha_i| \leq D } } |\alpha_1|! \ldots |\alpha_k|! \\
        &\leq D^{|\alpha_1|+\ldots+|\alpha_k|} \cdot \#\big\{ \alpha_1+ \ldots +\alpha_k = 2\gamma: \alpha_i \neq \mathbb O_d, |\alpha_i| \leq D \big\} \\
        &\leq D^{2|\gamma|} \cdot k^{2|\gamma|} \,.
    \end{align*}
    Plugging this estimation into \eqref{eq-m=1-relax-2}, we get that
    \begin{align*}
        \eqref{eq-m=1-relax-2} &\leq \sum_{ \gamma \neq \mathbb O_d,|\gamma|\leq D/2 } d^{-2|\gamma|} (D^4 k^2)^{|\gamma|} \leq \sum_{k=1}^{D/2} d^{-2k} (D^4 k^2)^{k} \cdot \#\big\{ \gamma \neq \mathbb O_d: |\gamma|=k \big\} \\
        &\leq \sum_{k=1}^{D/2} d^{-2k} (D^4 k^2)^{k} \cdot \tbinom{d+k}{k} \leq \sum_{k=1}^{D/2} d^{-k} (D^4 k)^{k} = [1+o(1)] \cdot \tfrac{D^4 k}{d} \,,
    \end{align*}
    where the last transition follows from $D^4k \leq D^5 =o(d)$. Thus, we get from \eqref{eq-Adv-bound-m=1} that
    \begin{align*}
        \mathsf{Adv}_{\leq D}(\Pb,\Qb)^2 \leq 1+ \sum_{1 \leq k \leq D/2} [1+o(1)] \cdot \tfrac{D^4 k}{d} = [1+o(1)] \cdot \tfrac{D^6}{d} = 1+o(1) \,,
    \end{align*}
    as desired.
\end{proof}

\section{Preliminary results}

\subsection{Preliminary results on Hermite polynomials}

\begin{lemma}{\label{lem-expan-inner-product-in-Hermite}}
    For $x,y \in \mathbb R^d$ with $\|y\|=1$ and $m \in \mathbb N$, we have
    \begin{equation}{\label{eq-expan-inner-product-in-Hermite}}
        \mathcal H_m(\langle x,y \rangle) = \sum_{ \alpha \in \mathbb N^d: |\alpha|=m } \sqrt{ \frac{ \alpha! }{ m! } } \binom{m}{\alpha} y^{\alpha} \mathcal H_{\alpha}(x) \,.
    \end{equation}
\end{lemma}
\begin{proof}
    Recall Definition~\ref{def-Hermite-poly}. It was shown in \cite{MOS13} that 
    \begin{align*}
        e^{tx-\frac{t^2}{2}} = \sum_{n=0}^{\infty} \frac{H_n(x)}{n!} \cdot t^n \,.
    \end{align*}
    Thus, we see that
    \begin{align*}
        \sum_{n=0}^{\infty} \frac{H_n(\langle x,y \rangle)}{n!} \cdot t^n &= e^{\langle x,y \rangle t-\frac{t^2}{2}} = e^{\langle x,y \rangle t-\frac{\| y \|^2 t^2}{2}} = \prod_{i=1}^{d} e^{x_iy_i t-\frac{y_i^2 t^2}{2}} \\ 
        &= \prod_{i=1}^{d} \Bigg( \sum_{n=0}^{\infty} \frac{H_n(x_i)}{n!} \cdot (y_i t)^n \Bigg) = \sum_{n=0}^{\infty} t^n \sum_{\alpha\in\mathbb N^d:|\alpha|=n} \frac{y^{\alpha} H_{\alpha}(x)}{\alpha!} \,.
    \end{align*}
    Thus, we see that
    \begin{align*}
        H_n(\langle x,y \rangle) = \sum_{\alpha\in\mathbb N^d:|\alpha|=n} \binom{n}{\alpha} y^{\alpha}  H_{\alpha}(x) \,.
    \end{align*}
    Combined with Definition~\ref{def-Hermite-poly}, we get the desired result.
\end{proof}

\subsection{Preliminary results on Wishart constant}

The following result can be found in \cite{Gautschi59}.

\begin{lemma}[Gautschi's inequality]{\label{lem-Gautschi-inequality}}
    for all $x \in \mathbb R_{\geq 0}$ we have $x^{1-s} \leq \frac{\Gamma(x+1)}{\Gamma(x+s)} \leq (x+1)^{1-s}$. In particular, we have when $x\to\infty$
    \begin{align*}
        \frac{ \Gamma(x+\frac{1}{2}) }{ \Gamma(x) } = [1+O(\tfrac{1}{x})] \cdot \sqrt{x} \,. 
    \end{align*}
\end{lemma}

Based on Lemma~\ref{lem-Gautschi-inequality}, the following result provides a crude bound on $\omega(s,t)$. 

\begin{lemma}{\label{lem-bound-omega-ratio}}
    Recall the definition of $\omega(s,t)$ in \eqref{eq-def-omega(s,t)}. For $k^3=o(d)$ we have
    \begin{align*}
        \frac{ \omega(d,k) }{ \omega(d-2k,k) } = [1+o(1)] \cdot d^{k^2} \mbox{ and } \frac{ \omega(d-2k,k) }{ \omega(d,k) } = [1+o(1)] \cdot d^{-k^2} \,.
    \end{align*}
    In addition, for $k^3,m^3=o(d)$ we have 
    \begin{align*}
        \frac{ \omega(d,m) }{ \omega(d-k,m) } = [1+o(1)] \cdot d^{km/2} \mbox{ and } \frac{ \omega(d-k,m) }{ \omega(d,m) } = [1+o(1)] \cdot d^{-km/2} \,.
    \end{align*}
\end{lemma}
\begin{proof}
    Using \eqref{eq-def-omega(s,t)}, we get that
    \begin{align*}
        \frac{ \omega(d,k) }{ \omega(d-2k,k) } &= 2^{k^2} \prod_{j=1}^{k} \frac{ \Gamma(\frac{d-j+1}{2}) }{ \Gamma(\frac{d-2k-j+1}{2}) } = 2^{k^2} \prod_{j=1}^{k} \prod_{i=1}^{k} \Big( \frac{d-j+1}{2} - i \Big) \\
        &= \prod_{j=1}^{k} \prod_{i=1}^{k} \Big( [1+O(k/d)] \cdot d \Big) = [1+o(1)] \cdot d^{k^2} \,,
    \end{align*}
    where in the last equality we use $k^3=o(d)$. Similarly we can get that
    \begin{equation*}
        \frac{ \omega(d-2k,k) }{ \omega(d,k) } = [1+o(1)] \cdot d^{-k^2} \,. 
    \end{equation*}
    In addition, using \eqref{eq-def-omega(s,t)} again we have
    \begin{align*}
        \frac{ \omega(d,m) }{ \omega(d-k,m) } &= 2^{km/2} \cdot \prod_{j=1}^{m} \frac{ \Gamma(\frac{d-j+1}{2}) }{ \Gamma(\frac{d-k-j+1}{2}) } = 2^{km/2} \prod_{j=1}^{m} \prod_{i=1}^{k} \frac{ \Gamma(\frac{d-i-j+2}{2}) }{ \Gamma(\frac{d-i-j+1}{2}) } \\
        &= 2^{km/2} \prod_{j=1}^{m} \prod_{i=1}^{k} \Big( \big( 1+O(\tfrac{2}{d-i-j+1}) \big) \cdot \sqrt{(d-i-j+1)/2} \Big) = [1+o(1)] \cdot d^{km/2} \,, 
    \end{align*}
    where in the last equality we use $k^3,m^3=o(d)$. Similarly we can get that
    \begin{equation*}
        \frac{ \omega(d-k,m) }{ \omega(d,m) } = [1+o(1)] \cdot d^{-km/2} \,.  \qedhere
    \end{equation*}
\end{proof}

\subsection{Preliminary results in probability}

\begin{lemma}{\label{lem-moments-uniform-sphere}}
    For all $\{ \gamma_i\in\mathbb N: 1 \leq i \leq d \}$ we have
    \begin{align*}
        \mathbb E_{Q \sim \operatorname{Unif}(\mathbb S^{d-1})}\Big[ Q_1^{\gamma_1} \ldots Q_{d}^{\gamma_d} \Big] = \mathbf 1_{ \{ \gamma_1,\ldots,\gamma_d \text{ is even} \} } \cdot \tfrac{ \Gamma(\frac{d}{2}) (\gamma_1-1)!! \ldots (\gamma_d-1)!! }{ \Gamma( \frac{d+\gamma_1+\ldots+\gamma_d}{2} ) 2^{\frac{\gamma_1+\ldots+\gamma_d}{2}} } \,.
    \end{align*}
\end{lemma}
\begin{proof}
    Denote $X=(X_1,\ldots,X_d) \sim \mathcal N(0,\mathbb I_d)$. We the have $\| X \| \perp \frac{X}{\|X\|}$ and $\frac{X}{\|X\|} \sim \operatorname{Unif}(\mathbb S^{d-1})$. Thus, we have
    \begin{align*}
        & \mathbb E_{Q \sim \operatorname{Unif}(\mathbb S^{d-1})}\Big[ Q_1^{\gamma_1} \ldots Q_{d}^{\gamma_d} \Big] \cdot \mathbb E\Big[ \|X\|^{\gamma_1+\ldots+\gamma_d} \Big] \\
        =\ & \mathbb E\Big[ \big( \tfrac{X_1}{\|X\|} \big)^{\gamma_1} \ldots \big( \tfrac{X_d}{\|X\|} \big)^{\gamma_d} \Big] \cdot \mathbb E\Big[ \|X\|^{\gamma_1+\ldots+\gamma_d} \Big] \\
        =\ & \mathbb E\big[ X_1^{\gamma_1} \ldots X_d^{\gamma_d} \big] = \mathbf 1_{ \{ \gamma_1,\ldots,\gamma_d \text{ is even} \} } \cdot \prod_{1\leq i \leq d} (\gamma_i-1)!! \,.
    \end{align*}
    In addition, since the law of $\|X\|^2$ is a Gamma distribution with $d$ sample and parameter $\tfrac{1}{2}$, we have
    \begin{align*}
        \mathbb E\Big[ \|X\|^{\gamma_1+\ldots+\gamma_d} \Big] = \frac{ 2^{ \frac{\gamma_1+\ldots+\gamma_d}{2} } \Gamma( \frac{d+\gamma_1+\ldots+\gamma_d}{2} ) }{ \Gamma(\frac{d}{2}) } \,.
    \end{align*}
    Thus, we see that the desired result holds.
\end{proof}

\begin{lemma}[\cite{JW19}, Lemma 2.1]{\label{lem-denisty-submatrix-uniform-O(n)}}
   Let $\mathbf{\Gamma}_d$ be an $d*d$ random matrix which is uniformly distributed on the orthogonal group $O(n)$ and let $\mathbf{Z}_n$ be the upper-left $p*q$ submatrix of $\mathbf{\Gamma}_n$. If $p+q \leq n$ and $p \geq q$ then the joint density function of entries of $\mathbf Z_n$ is given by
   \begin{align*}
       f(Z) = \frac{ \omega(d-p,q) }{ \omega(d,q)(2\pi)^{pq/2} } \cdot \operatorname{det}( \mathbb I_q - Z^{\top} Z )^{ (n-p-q-1)/2 } \zeta(Z^{\top}Z) \,,
   \end{align*}
   where $\zeta(Z^{\top}Z)$ is the indicator function of the set that all $q$ eigenvalues of $Z^{\top}Z$ are in $[0,1]$. When $p<q$, the density of $\mathbf Z_n$ is obtained by interchanging $p$ and $q$ in the above Wishart constant.
\end{lemma}

\begin{lemma}{\label{lem-Gaussian-exponential-moment}}
    Let $Z$ be an $d*m$ matrix with entries are i.i.d.\ $\mathcal N(0,1)$ random variables. Then for all $\lambda>0$ and $A \in \mathbb R^{d*m}$ we have
    \begin{align*}
        \mathtt E\Big[ e^{ -\lambda \|Z\|_{\Fop}^2 + \langle A,Z \rangle } \Big] = (1+2\lambda)^{-dm/2} \cdot e^{ \|A\|_{\Fop}^2/2(1+2\lambda) } \,.
    \end{align*}
\end{lemma}
\begin{proof}
    Note that
    \begin{align*}
        \mathbb E_{X \sim \mathcal N(0,1)}\Big[ e^{-\lambda x^2 + ax} \Big] = \tfrac{1}{\sqrt{1+2\lambda}} e^{ \frac{a^2}{2(1+2\lambda)} } \,,
    \end{align*}
    we then have
    \begin{equation*}
        \mathtt E\Big[ e^{ -\lambda \|Z\|_{\Fop}^2 + \langle A,Z \rangle } \Big] = \prod_{ \substack{ 1 \leq i \leq d \\ 1 \leq j \leq m } } \tfrac{1}{\sqrt{1+2\lambda}} e^{ \frac{a_{i,j}^2}{2(1+2\lambda)} } = (1+2\lambda)^{-dm/2} \cdot e^{ \|A\|_{\Fop}^2/2(1+2\lambda) } \,. \qedhere
    \end{equation*}
\end{proof}

\begin{lemma}{\label{lem-Gaussian-quadratic-form-moment}}
    Let $k<d$ and let $Z$ be an $d*k$ matrix with entries are i.i.d. $\mathcal N(0,1)$ random variables. Then for any $d*d$ semidefinite matrix $A$ we have
    \begin{align*}
        \mathbb E\Big[ e^{ -\operatorname{tr}(Z^{\top} A Z) } \Big] = \operatorname{det}(\mathbb I_d+2A)^{-k/2} \,.
    \end{align*}
\end{lemma}
\begin{proof}
    Note that for all $U \in O(d)$ we have
    \begin{align*}
        \mathbb E\Big[ e^{ -\operatorname{tr}(Z^{\top} A Z) } \Big] = \mathbb E\Big[ e^{ -\operatorname{tr}(Z^{\top} U^{\top} A U Z) } \Big] \,.
    \end{align*}
    Thus, without loss of generality we may assume that $A=\mathrm{Diag}(a_1,\ldots,a_d)$ is a diagonal matrix. In this case, we get from direct calculation that
    \begin{equation*}
        \mathbb E\Big[ e^{ -\operatorname{tr}(Z^{\top} A Z) } \Big] = \prod_{1 \leq i \leq d} \mathbb E\Big[ e^{-a_i\sum_{j=1}^{k} Z_{i,j}^2 } \Big] = \prod_{1 \leq i \leq d} (1+2a_i)^{-k/2} = \operatorname{det}(\mathbb I_d+2A)^{-k/2} \,. \qedhere
    \end{equation*}
\end{proof}

\begin{lemma}{\label{lem-bound-uniform-orthogonal-integration}}
    Denote $\nu=\nu_d$ to be the Haar measure over $O(d)$. For all $|\epsilon k|=O(1)$ and $|\epsilon| \leq 1-\Omega(1)$, we have as $d \to \infty$
    \begin{align*}
        \mathbb E_{Q \sim \nu}\Big[ \operatorname{det}(\mathbb I_d+\epsilon Q)^{k} \Big] = 1 + O(\epsilon k) \,.
    \end{align*}
\end{lemma}
\begin{proof}
    Clearly it suffices to show that when $|\epsilon k|=\Theta(1)$ we have $\mathbb E_{Q \sim \nu}\big[ \operatorname{det}(\mathbb I_d+\epsilon Q)^{k} \big] = O(1)$, and the general cases follows from the simple inequality $\mathbb E[X] \leq \mathbb E[X^m]^{\frac{1}{m}}$ for $m \geq 1$ and $X \geq 0$. Denote $\{ e^{i\theta_1}, \ldots, e^{i\theta_d}: \theta_1, \ldots, \theta_d \in [0,2\pi] \}$ to be the spectrum of $Q$. In addition, denote $\overline{F}_d$ to be the empirical distribution of $\theta_1,\ldots,\theta_d$. We then have
    \begin{align*}
        \operatorname{det}(\mathbb I_d+\epsilon Q)^{k} = \exp\Big( k \sum_{\ell=1}^{d} \log(1+\epsilon e^{i\theta_{\ell}}) \Big) = \exp\big( d \cdot W_d \big) \,,
    \end{align*}
    where 
    \begin{align*}
        W_d = \int_{0}^{2\pi} k\log(1+\epsilon e^{it}) \mathrm{d}\overline{F}_d(t) \,.
    \end{align*}
    Denote $\varphi(t)=k\log(1+\epsilon e^{it})$. Note that when $|\epsilon k|=O(1)$ and $|\epsilon| \leq 1 -\Omega(1)$ we have $|\varphi'(t)|\leq \frac{|\epsilon k|}{1-\epsilon} = \Theta(1)$ and thus $\varphi$ is 1-Lipchitz. 
    It was well known (see, e.g., \cite[Theorem~4.4.27]{AGZ10} and \cite[Theorem~2.1]{MM13}) that $\overline{F}_d$ weakly converges to $F$, the uniform distribution over $[0,2\pi]$. In addition, we have the following bound for the tail probability of $W_d$: 
    \begin{align*}
        \mathbb P\Big( \big| W_d \big| \geq t \Big) = \mathbb P\Big( \Big| W_d - k\int_{0}^{2\pi} \log(1+\epsilon e^{it}) \mathrm{d}F(t) \Big| \geq t \Big) \leq 2 e^{ -\Theta(1) \cdot d^2 t^2 } \,,
    \end{align*}
    where the equality follows from 
    \begin{align*}
        \int_{0}^{2\pi} \log(1+\epsilon e^{it}) \mathrm{d}F(t) = 0 \mbox{ for all } |\epsilon|<1
    \end{align*}
    and the inequality follows from \cite[Corollary~2.4]{MM13}. Thus, we have
    \begin{equation*}
        \mathbb E_{Q \sim \nu}\Big[ \operatorname{det}(\mathbb I_d+\epsilon Q)^{k} \Big] \leq 2 \int_{0}^{\infty} d e^{dt} \cdot e^{ -\Theta(1) \cdot d^2 t^2 } \mathrm{d}t = e^{ O(1) } = O(1) \,. \qedhere
    \end{equation*}
\end{proof}

\bibliographystyle{alpha}

\end{document}